%% file: ipBoost_arXiv.tex
\documentclass[11pt]{article}

\usepackage[utf8]{inputenc} % allow utf-8 input
\usepackage[T1]{fontenc}    % use 8-bit T1 fonts
\usepackage{mathpazo,tgcursor,tgpagella}
\usepackage{booktabs}       % professional-quality tables
\usepackage{amsfonts}       % blackboard math symbols
\usepackage{nicefrac}       % compact symbols for 1/2, etc.
\usepackage{microtype}      % microtypography
\usepackage{siunitx}        % for typesetting numbers
\usepackage{enumitem}       % for nice enumerate/itemize typesetting
\usepackage{dsfont}         % for nicer blackboard bold fonts
\usepackage{etoolbox}
\usepackage{algorithm}      % for algorithms
\usepackage{algorithmic}
\usepackage{graphicx}
\usepackage{color}
\usepackage{amssymb}
\usepackage{mathtools}
\usepackage[margin=1in]{geometry}
\mathtoolsset{showonlyrefs}

\usepackage[colorlinks,citecolor=blue,linkcolor=blue,urlcolor=blue]{hyperref}       % hyperlinks

\renewcommand{\algorithmiccomment}[1]{\bgroup\hfill\tiny//~#1\egroup}

\newcommand{\truth}[1]{\ensuremath{\mathbb{I}[#1]}}
\newcommand{\probability}[1]{\ensuremath{\mathbb{P}[#1]}}
\newcommand{\assign}{\leftarrow}  % assignment in an algorithm

\newcommand{\R}{\mathds{R}}
\newcommand{\Z}{\mathds{Z}}
\newcommand{\D}{\mathcal{D}}

\DeclarePairedDelimiter{\card}{\lvert}{\rvert}

\DeclarePairedDelimiterX{\normSimple}[1]{\lVert}{\rVert}
{\ifblank{#1}{\mathord{\cdot}}{#1}}

\newcommand{\tvtText}[1]{Train vs. test performance for
    different margins \(\rho\) on instance \texttt{#1}. Each point represents a
    solution encountered by IPBoost while solving the boosting
    problem. Grayscale values indicate the number of base learners used in
    boosted learners; see the legend.}

\newtheorem{lemma}{Lemma}

\newenvironment{proof}{\paragraph{Proof.}}{\hfill$\square$}

\title{IPBoost -- Non-Convex Boosting via Integer Programming}
\author{Marc E. Pfetsch\thanks{Department of Mathematics, TU Darmstadt,
    Germany, pfetsch@opt.tu-darmstadt.de} 
  \and Sebastian Pokutta\thanks{Department of Mathematics, TU Berlin and Zuse Institute Berlin, Berlin,
 Germany; pokutta@zib.de}}

\begin{document}

\maketitle

\begin{abstract}
  Recently non-convex optimization approaches for solving machine learning
  problems have gained significant attention. In
  this paper we explore non-convex boosting in classification by means of integer
  programming and demonstrate real-world practicability of the approach
  while circumventing shortcomings of convex boosting approaches. We
  report results that are comparable to or better than the current
  state-of-the-art.
\end{abstract}

\section{Introduction}
\label{sec:introduction}

Boosting is an important (and by now standard) technique in classification to combine
several \lq{}low accuracy\rq{} learners, so-called \emph{base
  learners}, into a \lq{}high accuracy\rq{} learner, a so-called
\emph{boosted learner}. Pioneered by the AdaBoost approach of
\cite{freund1995desicion}, in recent decades there has been extensive
work on boosting procedures and analyses of their limitations. In a nutshell,
boosting procedures are (typically) iterative schemes that roughly
work as follows: for \(t = 1, \dots, T\) do the following:
\begin{enumerate}
\item Train a learner \(\mu_t\) from a given class of base learners on
the data distribution \(\mathcal D_t\).
\item Evaluate performance of \(\mu_t\) by computing its loss.
\item Push weight of the data distribution \(\mathcal D_t\) towards
  misclassified examples leading to \(\mathcal D_{t+1}\).
\end{enumerate}
Finally, the learners are combined by some form of voting (e.g.,
soft or hard voting, averaging, thresholding). A
close inspection of most (but not all) boosting procedures reveals
that they solve an underlying convex optimization problem over a
convex loss function by means of coordinate gradient descent.
Boosting schemes of this type are often referred to as \emph{convex
  potential boosters}. These procedures can achieve exceptional
performance on many data sets if the data is correctly labeled.
However, it was shown in \cite{long2008random,long2010random} that any
convex potential booster can be easily defeated by a very small amount
of label noise (this also cannot be easily fixed by early
termination). The intuitive reason for this is that convex boosting
procedures might progressively zoom in on the (small percentage of)
misclassified examples in an attempt to correctly label them, while
simultaneously moving distribution weight away from the correctly
labeled examples. As a consequence, the boosting procedure might fail
and produce a boosted learner with arbitrary bad performance on
unseen data.

Let
\(\D = \{(x_i,y_i) \mid i \in I\} \subseteq \R^d \times \{\pm 1\}\) be
a set of training examples and for some logical condition~\(C\),
define \(\truth{C} = 1\) if \(C\) is true and \(\truth{C} = -1\)
otherwise. Typically, the true loss function of interest is of a form
similar to
\begin{align}
  \label{eq:trueLoss}
  \ell(\D, \theta) \coloneqq \sum_{i \in I} \truth{h_\theta(x_i) \neq
      y_i},
\end{align}
i.e., we want to minimize the number of misclassifications, where
\(h_\theta\) is some learner parameterized by \(\theta\); this
function can be further modified to incorporate margin maximization as
well as include a measure of complexity of the boosted learner to
help generalization etc. It is important to observe that the loss in
Equation~\eqref{eq:trueLoss} is non-convex and hard to minimize. Thus,
traditionally this loss has been replaced by various convex
relaxations, which are at the core of most boosting procedures. In
the presence of mislabeled examples (or more generally label noise)
minimizing these convex relaxations might not be a good proxy for minimizing
the true loss function arising from misclassifications.

Going beyond the issue of label noise, one might ask more broadly, why
not directly minimizing misclassifications (with possible
regularizations) if one could? In the past, this has been out of the
question due to the high complexity of minimizing the non-convex loss
function. In this paper, we will demonstrate that this is \emph{feasible and
practical} with today's integer programming techniques. We propose to
directly work with a loss function of the form as given in
\eqref{eq:trueLoss} (and variations) and solve the non-convex
combinatorial optimization problem with state-of-the-art integer programming (IP)
techniques including column generation. This approach generalizes
previous linear programming based approaches (and hence implicitly
convex approaches) in, e.g.,
\cite{demiriz2002linear,goldberg2010boosting,goldberg2012sparse,eckstein2012improved},
while solving classification problems with the true misclassification
loss. We acknowledge that \eqref{eq:trueLoss} is theoretically very hard
(in fact NP-hard as shown, e.g., in \cite{goldberg2010boosting}),
however, we hasten to stress that in real-world computations for
\emph{specific instances} the behavior is often much better than the
theoretical asymptotic complexity. In fact, most real-world instances
are actually relatively ``easy'' and with the availability of very
strong integer programming solvers such as, e.g., the commercial solvers
\texttt{CPLEX}, \texttt{Gurobi}, and \texttt{XPRESS} and the academic
solver \texttt{SCIP}, these problems can be often solved rather
effectively. In fact, integer programming methods have seen a huge
improvement in terms of computational speed as reported in
\cite{savicky2000optimal,bertsimas2017optimal}. The latter reports
that integer programming solving performance has seen a combined
hardware and software speed-up of \(80\) billion from 1991 to 2015
(hardware: \num{570000}, software \num{1400000}) using state-of-the-art
hardware and solvers such as \texttt{CPLEX} (see \cite{cplex}),
\texttt{Gurobi} (see \cite{gurobi}), \texttt{XPRESS} (see \cite{xpress}),
and SCIP (see \cite{GleixnerEtal2018OO}). With this, problems that
traditionally have been deemed unsolvable can be solved in reasonable
short amounts of time making these methods accessible, feasible, and
practical in the context of machine learning allowing to solve a
(certain type of) non-convex optimization problems.

\subsection*{Contribution and Related Work}
\label{sec:contribution} Our contribution can be summarized as follows:

\emph{IP-based boosting.} We propose an integer programming based
boosting procedure. The resulting procedure utilizes column generation
to solve the initial learning problem and is inherently robust to labeling
noise, since we solve the problem for the (true) non-convex loss function. In
particular, our procedure is robust to the instances from
\cite{long2008random,long2010random} that defeat other convex
potential boosters.

Linear Programming (LP) based boosting procedures
have been already explored with \emph{LPBoost}
\cite{demiriz2002linear}, which also relies on column generation to
price the learners. Subsequent work in \cite{leskovec2003linear}
considered LP-based boosting for uneven datasets. We also perform
column generation, however, in an IP framework (see
\cite{desrosiers2005primer} for an introduction) rather than a purely
LP-based approach, which significantly complicates things. In order to control complexity, overfitting, and
generalization of the model typically some sparsity is
enforced. Previous approaches in the context of LP-based boosting have
promoted sparsity by means of cutting planes, see, e.g.,
\cite{goldberg2010boosting,goldberg2012sparse,eckstein2012improved}.
Sparsification can be handled in our approach by solving a delayed
integer program using additional cutting planes.

An interesting alternative use of boosting in the context of training
average learners against rare examples has been explored in
\cite{shalev2016minimizing}; here the \lq{}boosting\rq{} of the data
distribution is performed \emph{while} a more complex learner is
trained. In \cite{freund2015new} boosting in the context of
linear regression has been shown to reduce to a certain form of
subgradient descent over an appropriate loss function. For a general
overview of boosting methods we refer the interested reader to
\cite{schapire2003boosting}. Non-convex approaches to machine
learning problems gained recent attention and (mixed) integer programming, in particular,
has been used successfully to incorporate combinatorial structure in
classification, see, e.g.,~\cite{bertsimas2007classification,chang2012integer,bertsimas2015or,bertsimas2016best},
as well as,
\cite{gunluk2016optimal,bertsimas2017optimal,dash2018boolean,gunluk2018optimal,verwer2019learning};
note that \cite{dash2018boolean} also uses a column generation approach.
Moreover, neural network verification via integer
programming has been treated in
\cite{tjeng2017evaluating,fischetti2018deep}. See also the references
contained in all of these papers.

\emph{Computational results.} We present computational results
demonstrating that IP-based boosting can avoid the bad examples
of~\cite{long2008random}: by far better solutions can be obtained via
LP/IP-based boosting for these instances. We also show that IP-based
boosting can be competitive for real-world instances from the LIBSVM data set. In
fact, we obtain nearly optimal solutions in reasonable
time for the true non-convex cost function. Good solutions can be obtained
if the process is stopped early. While it cannot match the
raw speed of convex boosters, the obtained results are (often) much
better. Moreover, the resulting solutions are often sparse.

\section{IPBoost: Boosting via Integer Programming}
\label{sec:boosting-via-integer}

We will now introduce the basic formulation of our boosting problem,
which is an integer programming formulation based on the standard LPBoost
model from \cite{demiriz2002linear}. While we confine the
exposition to the binary classification case only, for the sake of
clarity, we stress that our approach can
be extended to the multi-class case using standard methods. In
subsequent sections, we will refine the model to include additional
model parameters etc.

Let \((x_1, y_1), \dots, (x_N, y_N)\) be the training set with points
\(x_i \in \R^d\) and two-class labels \(y_i \in \{\pm 1\}\). Moreover, let
\(\Omega \coloneqq \{h_1, \dots, h_L: \R^d \to \{\pm 1\}\}\) be a
class of base learners and let a margin \(\rho \geq 0\) be given. Our
basic boosting model is captured by the following integer programming
problem: 
\begin{align}\label{mod:baseIP}
  \min\; & \sum_{i=1}^N z_i \\
  & \sum_{j=1}^L \eta_{ij}\, \lambda_j + (1 + \rho) z_i \geq
  \rho\quad\forall\, i \in [N],\\
  & \sum_{j=1}^L \lambda_j = 1,\; \lambda \geq 0,\\
  & z \in \{0,1\}^N,
\end{align}
where the \emph{error function} \(\eta\) can take various forms
depending on how we want to treat the output of base learners. For
learner \(h_j\) and training example \(x_i\) we consider the following choices:
\begin{enumerate}[leftmargin=5ex,label=\emph{(\(\roman*\))}]
\item \label{item:1}\(\pm 1\) classification from learners:\\
  \(\eta_{ij} \coloneqq 2\, \truth{h_j(x_i) = y_i} -1 = y_i \cdot h_j(x_i)\);
\item \label{item:2} class probabilities of learners:\\
  \(\eta_{ij} \coloneqq 2\,\probability{h_j(x_i) = y_i} -1 \);
\item \label{item:3} \texttt{SAMME.R} error function for learners:\\
  \(\eta_{ij} \coloneqq \tfrac{1}{2} y_i \log\Big(\frac{\probability{h_j(x_i) = 1}}{\probability{h_j(x_i) = -1}}\Big)\).
\end{enumerate}
In the first case we perform a hard minimization of the classification
error, in the second case we perform a soft minimization of
the classification error, and in the last one we minimize the
\texttt{SAMME.R} error function as used in the (multi-class) AdaBoost
variant in~\cite{zhu2009multi}. The \texttt{SAMME.R}
error function allows a very confident learner to overrule a larger
number of less confident learners predicting the opposite class. 

The \(z_i\) variable in the model above indicates whether example
\(i \in [N] \coloneqq \{1, \dots N\}\) satisfies the classification
requirement: \(z_i = 0\) if example \(i\) is correctly labeled by the
boosted learner \(\sum_{j} h_j \lambda_j\) with margin at least
\(\rho\) with respect to the utilized error function \(\eta\);
in an optimal solution, if a variable if 1 this implies misclassification,
otherwise by minimizing you could have set it to zero.
The \(\lambda_j\) with \(j \in [L]\) form a distribution over the
family of base learners. The only non-trivial family of inequalities
in~\eqref{mod:baseIP} ranges over examples \(i \in [N]\) and enforces
that the combined learner \(\sum_{j \in [L]} h_j \lambda_j\)
classifies example \(i \in N\) correctly with margin at least \(\rho\)
(we assume throughout that \(\rho \leq 1\))
or \(z_i = 1\), i.e., the example is disregarded and potentially
misclassified. By minimizing \(\sum_{i \in N} z_i\), the program
computes the best combination of base learners maximizing the number
of examples that are correctly classified with margin at least
\(\rho\). The margin parameter \(\rho\) helps generalization as it
prevents base learners to be used to explain low-margin noise.

Before we continue with the integer programming based boosting
algorithm we would like to remark the following about the solution
structure of optimal solutions with respect to the chosen margin:

\begin{lemma}[Structure of high-margin solutions]\mbox{}\\ 
  Let \((\lambda,z)\) be an optimal solution to the integer
  program~\eqref{mod:baseIP} for a given margin \(\rho\) using error
  function~\ref{item:1}. Further let
  \(I \coloneqq \{i \in [N] \mid z_i = 0\}\) and
  \(J \coloneqq \{j \in [L] \mid \lambda_j > 0\}\). If the optimal
  solution is non-trivial, i.e., \(I \neq \emptyset\), then the
  following holds:
  \begin{enumerate}[leftmargin=4ex,itemsep=0ex]
  \item If \(\rho = 1\), then there exists an optimal solution with
    margin \(1\) using
    only a single base learner \(h_j\) for some \(j \in J\).
  \item If there exists \(\bar{\jmath} \in J\) with
    \(\lambda_{\bar{\jmath}} > \frac{1 - \rho}{2}\), then \(h_{\bar{\jmath}}\) by itself is already an
    optimal solution with margin~\(1\).
  \item If \(\card{J} < \frac{2}{1-\rho}\), then there exists
    \(\bar{\jmath} \in J\) with \(h_{\bar{\jmath}}\) by itself being already an
    optimal solution with margin~\(1\). In particular for \(\rho >
    0\), the statement is non-trivial. 
  \end{enumerate}
\end{lemma}
\begin{proof}
  For the first case observe that 
  \[
  \sum_{j \in J} \eta_{ij} \lambda_j \geq 1,
  \] 
  holds for all \(i \in I\). As \(\sum_{j \in J} \lambda_j = 1\) and
  \(\lambda_j > 0\) for all \(j \in J\), we have that \(\eta_{ij} = 1\)
  for all \(i \in I\), \(j \in J\). Therefore the predictions of all
  learners \(h_j\) with \(j \in J\) for examples \(i \in I\) are
  identical and we can simply choose any such learner \(h_j\) with \(j
  \in J\) arbitrarily and set \(\lambda_j = 1\). 

  For the second case observe as before that
  \(
  \sum_{j \in J} \eta_{ij} \lambda_j \geq \rho
  \) 
  holds for all \(i \in I\). We claim that \(\eta_{i\bar{\jmath}} = 1\) for
  all \(i \in I\). For contradiction suppose not, i.e., there exists
  \(\bar{\imath} \in I\) with \(\eta_{\bar{\imath} \bar{\jmath}} = -1\). 
  Then 
  \[
  \sum_{j \in J} \eta_{\bar{\imath}j} \lambda_j < \underbrace{\bigg( \sum_{j \in J
    \setminus \{\bar{\jmath}\}} \lambda_j  \bigg)}_{< 1- \frac{1-\rho}{2}} - \frac{1-\rho}{2} < \rho,
  \]
  using \(\eta_{ij} \leq 1\), \(\sum_{j \in J} \lambda_j = 1\), and
  \(\lambda_j > 0\) for all \(j \in J\). This contradicts \(\sum_{j \in
    J} \eta_{ij} \lambda_j \geq \rho\) and therefore \(\eta_{i\bar{\jmath}} = 1\) for
  all \(i \in I\). Thus \(h_{\bar{\jmath}}\) by itself is already an optimal
  solution satisfying even the (potentially) higher margin of \(1 \geq \rho\) on
  examples \(i \in I\).

  Finally, for the last case observe that if \(\card{J} <
  \smash{\frac{2}{1-\rho}}\), then together with \(\sum_{j \in J} \lambda_j = 1\),
  and \(\lambda_j > 0\) for all \(j \in J\), it follows that there exists
  \(\bar{\jmath} \in J\) with \(\lambda_{\bar{\jmath}} > \smash{\frac{1 - \rho}{2}}\). Otherwise
  \(\sum_{j \in J} \lambda_j \leq \card{J} \smash{\frac{1 - \rho}{2} < 1}\); a
  contradiction. We can now apply the second case.
\end{proof}

Similar observations hold for error functions \ref{item:2} and
\ref{item:3} with the obvious modifications to include the actual
value of \(\eta_{ij}\) not just its sign. 

Our proposed solution process consists of two parts. We first solve the integer
program in \eqref{mod:baseIP} using column generation. Once
this step is completed, the solution can be sparsified (if necessary) by means of the
model presented in Section~\ref{sec:sparsification}, where we
trade-off classification performance with model complexity.

\subsection{Solution Process using Column Generation}
\label{sec:column-generation}

The reader will have realized that~\eqref{mod:baseIPdual} is not
practical, since we typically have a very large if not infinite
class of base learners \(\Omega\); for convenience we assume here that
\(\Omega\) is finite but potentially very large. This has been already observed
before and dealt with effectively via column generation in
\cite{demiriz2002linear,goldberg2010boosting,goldberg2012sparse,eckstein2012improved}.
We will follow a similar strategy here, however, we generate columns
\emph{within} a branch-and-bound framework leading effectively to a
\emph{branch-and-bound-and-price} algorithm that we are using; this is
significantly more involved compared to column generation in linear
programming. We detail this approach in the following.

The goal of column generation is to provide an efficient way to solve the
linear programming (LP) relaxation of~\eqref{mod:baseIP}, i.e., the \(z_i\)
variables are relaxed and allowed to assume fractional values. Moreover,
one uses a subset of the columns, i.e., base learners, \(\mathcal{L}
\subseteq [L]\). This yields the so-call \emph{restricted master (primal)
  problem}
\begin{align}
  \min\; & \sum_{i=1}^N z_i\label{mod:restrictedMaster}\\
  & \sum_{j \in \mathcal{L}} \eta_{ij}\, \lambda_j + (1 + \rho) z_i \geq
  \rho\quad\forall\, i \in [N],\\
  & \sum_{j \in \mathcal{L}} \lambda_j = 1,\; \lambda \geq 0,\; z \in [0,1]^N.
\end{align}
Its \emph{restricted dual problem} is
\begin{align}
  \max\; & \rho \sum_{i=1}^N w_i + v - \sum_{i=1}^N u_i\label{mod:baseIPdual}\\
  & \sum_{i=1}^N \eta_{ij}\, w_i + v \leq 0 \quad\forall\, j \in \mathcal{L},\\
  & \;(1 + \rho) w_i - u_i \leq 1\quad\forall\, i \in [N],\\
  & w \geq 0,\; u \geq 0,\; v \text{ free}.
\end{align}
Consider a solution \((w^*, v^*, u^*) \in \R^N \times \R \times \R^N\) of~\eqref{mod:baseIPdual}. The
so-called \emph{pricing problem} is to decide whether this solution is
actually optimal or whether we can add further constraints, i.e., columns
in the primal problem. For this, we need to check whether \((w^*, v^*,
u^*)\) is feasible for the complete set of constraints
in~\eqref{mod:baseIPdual}. In the following, we will assume that the
variables~\(z_i\) are always present in the primal and therefore that the
corresponding inequalities \((1 + \rho) w_i - v \leq 1\) are satisfied for
each \(i \in [N]\). Thus, the main task of the pricing problem is to decide
whether there exists \(j \in [L] \setminus \mathcal{L}\) such that
\begin{equation}\label{eq:PricingProb}
  \sum_{i=1}^N \eta_{ij}\, w_i^* + v^* > 0.
\end{equation}
If such an \(j\) exists, then it is added to~\(\mathcal{L}\), i.e.,
to~\eqref{mod:restrictedMaster}, and the process is iterated. Otherwise,
both~\eqref{mod:restrictedMaster} and~\eqref{mod:baseIPdual} have been
solved to optimality.

The pricing problem~\eqref{eq:PricingProb} can now be rephrased as
follows: Does there exist a base learner~\(h_j \in \Omega\) such
that~\eqref{eq:PricingProb} holds? For this, the \(w_i^*\) can be seen
as weights over the points \(x_i\), \(i \in [N]\), and we have to
classify the points according to these weights. For most base
learners, this task just corresponds to an ordinary classification or
regression step, depending on the form chosen for \(\eta_{ij}\). Note,
however, that in practice~\eqref{eq:PricingProb} is not solved to
optimality, but is rather solved heuristically. If we find a base
learner \(h_j\) that satisfies~\eqref{eq:PricingProb}, we continue,
otherwise we stop.

\begin{algorithm}[tb]
  \caption{IPBoost}
  \label{alg:ipBoost}
  \begin{algorithmic}[1]
    \REQUIRE Examples
    \(\D = \{(x_i,y_i) \mid i \in I\} \subseteq \R^d \times
    \{\pm 1\}\), class of base learners \(\Omega\), margin \(\rho\)
    \ENSURE Boosted learner \(\sum_{j \in \mathcal{L}^*} h_j \lambda_j^*\) with base learners \(h_j\) and weights \(\lambda_j^*\)
    \STATE \(\mathcal{T} \assign \{([0,1]^N, \varnothing)\}\) \COMMENT{set of local bounds and learners for open subproblems}
    \STATE \(U \assign \infty\), \(\mathcal{L}^* \assign \varnothing\) \COMMENT{Upper bound on optimal objective.}
    \WHILE{\(\mathcal{T} \neq \varnothing\)}
    \STATE Choose and remove \((B,\mathcal{L})\) from \(\mathcal{T}\).
    \REPEAT
    \STATE Solve~\eqref{mod:restrictedMaster} using the local bounds on \(z\) in \(B\) with optimal dual solution \((w^*,v^*,u^*)\).
    \STATE Find learner \(h_j \in \Omega\) satisfying~\eqref{eq:PricingProb}.\COMMENT{Solve pricing problem.}
    \UNTIL{\(h_j\) is not found}
    \STATE Let \((\tilde{\lambda}, \tilde{z})\) be the final solution
       of~\eqref{mod:restrictedMaster} with base learners \(\tilde{\mathcal{L}} = \{j \mid \tilde{\lambda}_j > 0\}\).
    \IF{\(\tilde{z} \in \Z^N\) and \(\sum_{i=1}^N \tilde{z_i} < U\)}
    \STATE \(U \assign \sum_{i=1}^N \tilde{z}_i\), \(\mathcal{L}^* \assign \tilde{\mathcal{L}}\), \(\lambda^* \assign \tilde{\lambda}\)
       \COMMENT{Update best solution.}
    \ELSE
    \STATE Choose \(i \in [N]\) with \(\tilde{z}_i \notin \Z\).
    \STATE Set \(B_0 \assign B \cap \{z_i \leq 0\}\), \(B_1 \assign B \cap \{z_i \geq 1\}\).
    \STATE Add \((B_0, \tilde{\mathcal{L}})\), \((B_1, \tilde{\mathcal{L}})\) to \(\mathcal{T}\).
       \COMMENT{Create new branching nodes.}
    \ENDIF
    \ENDWHILE
    \STATE \emph{Optionally sparsify final solution \(\mathcal{L}^*\)}.
  \end{algorithmic}
\end{algorithm}

The process just described allows to solve the relaxation
of~\eqref{mod:baseIP}. The optimal misclassification values are determined
by a branch-and-price process that branches on the variables~\(z_i\) and
solves the intermediate LPs using column generation. Note that the \(z_i\)
are always present in the LP, which means that no problem-specific
branching rule is needed, see, e.g., \cite{BarJNSV98} for a discussion. In
total, this yields Algorithm~\ref{alg:ipBoost}.

The output of the algorithm is a set of base learners~\(\mathcal{L}^*\) and
corresponding weights \(\lambda_j^*\). A classifier can be obtained by
\emph{voting}, i.e., a given point~\(x\) is classified by each of the base
learners resulting in \(\xi_j\), for which we again can use the three
options \ref{item:1}--\ref{item:3} above. We then take the weighted
combination and obtain the predicted label as \(\text{sgn}(\sum_{j \in
  \mathcal{L}^*} \xi_j \lambda_j^*)\).

In the implementation, we use the following important components. First, we
use the framework SCIP that automatically applies primal heuristics, see,
e.g., \cite{Berthold14} for an overview. These heuristics usually take the
current solution of the relaxation and try to build a feasible solution
for~\eqref{mod:baseIP}. In the current application, the most important
heuristics are rounding heuristics, i.e., the \(z_i\) variables are rounded
to \(0\) or \(1\), but large-scale neighborhood heuristics sometimes
provide very good solutions as well. Nevertheless, we disable diving
heuristics, since these often needed a long time, but never produced a
feasible solution. In total, this often generates many feasible solutions
along the way.

Another trick that we apply is the so-called \emph{stall limit}. The solver
automatically stops if the best primal solution could not be (strictly)
improved during the last \(K\) nodes processed in the branch-and-bound tree
(we use \(K = 5000\)).

Furthermore, preliminary experiments have shown that the intermediate
linear programs that have to be solved in each iteration become
increasingly hard to solve by the simplex algorithm for a large number of
training points. We could apply bagging~\cite{Bre96}, but obtained good
results with just subsampling \num{30000} points if their number \(N\) is larger than
this threshold.

Furthermore, we perform the following post-processing. For the best solution
that is available at the end of the branch-and-bound algorithm, we fix the
integer variables to the values in this solution. Then we maximize the
margin over the learner variables that were used in the solution, which is
just a linear program. In most cases, the margin can be slightly improved
in this way, hoping to get improved generalization.

\subsection{Sparsification}
\label{sec:sparsification}

One of the challenges in boosting is to balance model accuracy vs.~model
generalization, i.e., to prevent overfitting. Apart from pure
generalization considerations, a sparse model often lends itself more
easily to interpretation, which might be important in certain applications.

There are essentially two techniques that are commonly used in this
context. The first one is early stopping, i.e., we only perform a fixed
number of boosting iterations, which would correspond to only generating a
fixed number of columns. The second common approach is to regularize the
problem by adding a complexity term for the learners in the objective
function, so that we minimize \( \sum_{i=1}^N z_i + \sum_{j=1}^L \alpha_j
y_j\). Then we can pick \(\alpha_j\) as a function of the complexity of the
learner \(h_j\). For example, in \cite{cortes2014deep} boosting across
classes of more complex learners has been considered and the \(\alpha_j\)
are chosen to be proportional to the Rademacher complexity of the learners
(many other measures might be equally justified).

In our context, it seems natural to consider the following integer
program for sparsification:
\begin{align}
  \min\; & \sum_{i=1}^N z_i + \sum_{j=1}^L \alpha_j y_j\label{mod:sparseIP} \\
  & \sum_{j=1}^L \eta_{ij}\, \lambda_j + (1 + \rho) z_i \geq
  \rho\quad\forall\, i \in [N],\\
  & \sum_{j=1}^L \lambda_j = 1,\; 0 \leq \lambda_j \leq y_j\quad\forall\, j \in [L],\\
  & z \in \{0,1\}^N,\; y \in \{0,1\}^L,
\end{align}
with \(\eta_{ij}\) as before. The structure of this sparsification
problem that involves additional binary variables~\(y\) cannot be easily
represented within the column generation setup used to solve
model~\eqref{mod:baseIP}, because the upper bounds on \(\lambda_j\)
implied by \(y_j\) would need to represented in dual problem, giving
rise to exponentially many variables in the dual. In principle, one
could handle a cardinality constraint on the \(y_j\) variables using a
problem specific branching rule; this more involved algorithm is
however beyond the scope of this paper. In consequence, one can
solve the sparsification problem separately for the columns that have
been selected in phase \(1\) once this phase is completed. This is
similar to~\cite{goldberg2010boosting}, but directly aims to solve
the MIP rather than a relaxation. Moreover, one can apply so-called IIS-cuts,
following~\cite{Pfetsch2008}. Using the Farkas lemma, the idea is to
identify subsets \(I \subseteq [N]\) such that the system
\[
\sum_{j=1}^L \eta_{ij}\, \lambda_j \geq \rho,\; i \in I,\quad \sum_{j=1}^L
\lambda_j = 1,\; \lambda \geq 0,
\]
is infeasible. In this case the cut
\[
\sum_{i \in I} z_i \geq 1
\]
is valid. Such sets~\(I\) can be found by searching for vertices of the
corresponding alternative polyhedron. If this is done iteratively
(see~\cite{Pfetsch2008}), many such cuts can be found that help to
strengthen the LP relaxation. These cuts dominate the ones
in~\cite{goldberg2010boosting}, but one needs to solve an LP for each
vertex/cut.

\section{Computational Results}
\label{sec:comp-results}

To evaluate the performance of IPBoost, we ran a wide range of tests on
various classification tasks. Due to space limitations, we will only be
able to report aggregate results here; additional more extensive results
can be found in the Supplementary
Material~\ref{sec:computational-tests}. 
% The code is also available through the Supplementary Material.

\paragraph{Computational Setup.}

All tests were run on a Linux cluster with Intel Xeon quad core CPUs
with 3.50GHz, 10 MB cache, and 32 GB of main memory. All runs were
performed with a single process per node; we stress, in particular,
that we run all tests as \emph{single thread / single core} setup,
i.e., each test uses only a single node in single thread mode. We used a
prerelease version of SCIP 7.0.0 with SoPlex 5.0.0 as LP-solver; note that
this combination is completely available in source code and free for
academic use. The main part of the code
was implemented in C, calling the python framework
\texttt{scikit-learn} \cite{pedregosa2011scikit} at several
places.
% Note that scikit-learn still uses the same citation request from 2011
We use the decision tree implementation of
\texttt{scikit-learn} with a maximal depth of 1, i.e., a decision stump, as base learners for
all boosters. We
benchmarked IPBoost against our own implementation of
LPBoost~\cite{demiriz2002linear} as well as the AdaBoost
implementation in version 0.21.3 of \texttt{scikit-learn} using 100
iterations; note that we always report the number of pairwise distinct base
learners for AdaBoost. We performed \(10\) runs for each instance with varying random seeds and
we report average accuracy and standard deviations.  Note that we use
a time limit of one hour for each run of IPBoost. The reported
solution is the best solution available at that time.

\paragraph{Results on Constructed Hard Instances.}

\input{tables/table-hard2-test}

We start our discussion of computational results by reporting on
experiments with the hard instances of~\cite{long2008random}.  These
examples are tailored to using the \(\pm 1\)
classification from learners (option \textit{(i)} in
Section~\ref{sec:boosting-via-integer}). Thus, we use this function for
prediction and voting for every algorithm. The performance of IPBoost,
LPBoost and AdaBoost (using 100 iterations) is presented in
Table~\ref{tab:hard2}. Here, \(N\) is the number of points and \(\gamma\)
refers to the noise level. Note that we randomly split off 20\,\% of the
points for the test set.

On every instance class, IPBoost clearly outperforms LPBoost. 
AdaBoost performs much less well, as expected; it also
uses significantly more base learners. Note, however, that
the \texttt{scikit-learn} implementation of AdaBoost produces much better
results than the one in~\cite{long2008random} (an accuracy of about 53\,\%
as opposed to 33\,\%). As noted above, the
instances are constructed for a \(\pm 1\) classification function. If we
change to \texttt{SAMME.R}, AdaBoost performs much better: slightly worse
that IPBoost, but better than LPBoost.

\paragraph{LIBSVM Instances.}
\label{sec:instances}

We use classification instances from the LIBSVM data sets available at
\url{https://www.csie.ntu.edu.tw/~cjlin/libsvmtools/datasets/}. We selected
the 40 smallest instances. If available, we choose the scaled
version over the unscaled version. Note that 25 instances of those 40
instances come with a corresponding test set. Since the test sets for the
instances \texttt{a1a}--\texttt{a9a} looked suspicious (often more features
and points than in the train set and sometimes only features in one class),
we decided to remove the test sets for these nine instances. This leaves 16
instances with test set. For the other 24, we randomly split off 20\% of
the points as a test set; we provide statistics for the individual instances in
Table~\ref{tab:instanceStats} in the Supplementary
Material~\ref{sec:computational-tests}.

\paragraph{Results for LIBSVM.}

An important choice for the algorithm is how the error matrix~\(\eta\) is
set up, i.e., which of the three options \ref{item:1}--\ref{item:3}
presented in Section~\ref{sec:boosting-via-integer} is used. In preliminary
computations, we compared all three possibilities. It turned out that the
best option is to use the class probabilities \ref{item:2} for \(\eta\) both for
Model~\eqref{mod:baseIP} and when using the base learners in a voting
scheme, which we report here.

\input{tables/table-libsvm-overview2}

Another crucial choice in our approach is the margin
bound~\(\rho\). We ran our code with different values -- the aggregated
results are presented in Table~\ref{tab:aggregated}; the detailed
results are given in the Supplementary Material. We report accuracies on the test set and
train set, respectively. In each case, we report the averages of the
accuracies over 10 runs with a different random seed and their
standard deviations. The accuracies of IPBoost are compared to LPBoost
and AdaBoost. We also report the number \(L\) of learners in the detailed results. Note that the behavior of
AdaBoost is independent of \(\rho\), i.e., the accuracies are the same over
the different values of \(\rho\) in Table~\ref{tab:aggregated}.

The results show that IPBoost outperforms both LPBoost and
AdaBoost. IPBoost clearly outperforms LPBoost, although there are
instances where LPBoost generates slightly better results, both for
the train and the test accuracies. Interestingly, the accuracies of
IPBoost (and LPBoost) increase with respect to AdaBoost, when
considering the test set instead of the training set: less overfitting and
better generalization. For the considered instances the best value for
the margin \(\rho\) was \(0.05\) for LPBoost and IPBoost; AdaBoost has
no margin parameter.

Depending on the size of the instances, typical running times of IPBoost
range from a few seconds up to one hour. We provide details of the running
times in Table~\ref{tab:runtimeStats} in the Supplementary Material. The
average run time of IPBoost for \(\rho = 0.05\)
is \num{1367.78} seconds, while LPBoost uses \num{164.35} seconds and
AdaBoost \num{3.59} seconds. The main bottleneck arises from the solution
of large LP relaxations in each iteration of the algorithm. Note that we
apply the simplex algorithm in order to benefit from hot start after
changing the problem by adding columns or changing bounds. Nevertheless,
larger LPs turned out to be very hard to solve. One explanation for this is
that the matrix is very dense.

Feasible solutions of high quality are often
found after a few seconds via primal heuristics. The solution that is
actually used for constructing the boosted learner is often found long
before the solution process finished, i.e., the algorithm continues to
search for better solutions without further progress. Note that in most
cases, the algorithm is stopped, because no further improving solution was
found, i.e., the stall limit is applied (see
Section~\ref{sec:column-generation}). We have experimented with larger
limits (\(K > 5000\)),
but the quality of the solutions only improved very slightly. This suggests
that the solutions we found are optimal or close to optimal
for~\eqref{mod:baseIP}.

Also interesting is the number of base learners in the best solutions of
IPBoost. The results show that this is around 12 on average for \(\rho =
0.05\); for \(\rho = 0.01\) it is around 18. Thus, the optimal solutions
are inherently sparse and as such for these settings and instances the
sparsification procedure described in Section~\ref{sec:sparsification} will
likely not be successful. However, it seems likely that for instance sets
requiring different margins the situation is different.

We have also experimented with different ways to
handle~\(\rho\). Following~\cite{demiriz2002linear}, one can set up a model
in which the margin~\(\rho\) is a variable to be optimized in addition to
the number of misclassifications. In this model, it is crucial to find the
right balance between the different parts of the objective. For instance,
on can run some algorithm (AdaBoost) to estimate the number of
misclassifications and then adjust the weight factor accordingly. In
preliminary experiments, this option was inferior to the approach described
in this paper; we used an identical approach for LPBoost for
consistency.

\paragraph{Generalization Performance.}

We found that the boosted learners computed via IPBoost generalize rather
well. Figure~\ref{fig:tvt_w1a} gives a representative example for
generalization: here we plot the train and test accuracy of the solutions
encountered by IPBoost within a run, while solving the boosting problem for
various margins. We report more detailed results in
Section~\ref{sec:gener-ipbo-train} in the Supplementary Material.

We observe the following almost monotonous behavior: the smaller
\(\rho\), the more base learners are used and the better the obtained
training accuracy. This is of course expected, since smaller margins
allow more freedom to combine base learners. However, this behavior does not
directly translate to better testing accuracies, which indices
overfitting. Note that IPBoost obtains better average test accuracies
than AdaBoost for every \(\rho\), but this is not always the case for
the train accuracies. This again demonstrates the good generalization
properties of IPBoost.

We would also like to point out that the results in
Figure~\ref{fig:tvt_w1a} and Section~\ref{sec:gener-ipbo-train}
give an indication that the often cited belief that ``solving (close) to
optimality reduces generalization'' is not true in general. In fact,
minimizing the right loss function close to optimality can actually
\emph{help} generalization.

\begin{figure}
  \centering
  \includegraphics[width=0.33\textwidth]{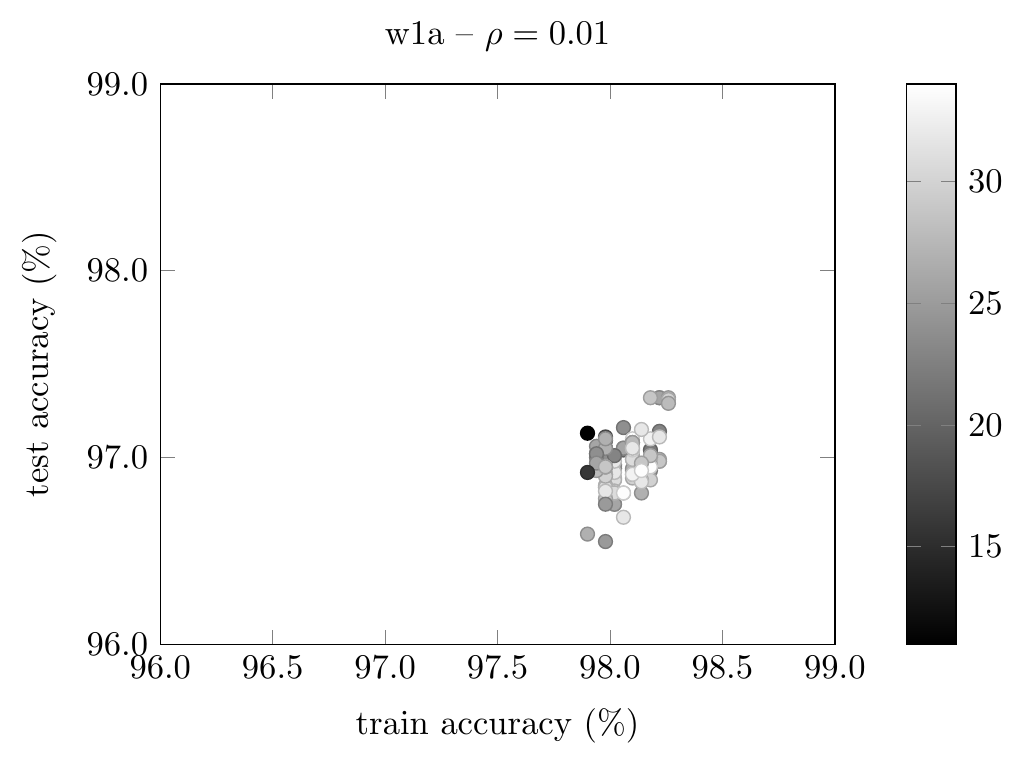}\hfill
  \includegraphics[width=0.33\textwidth]{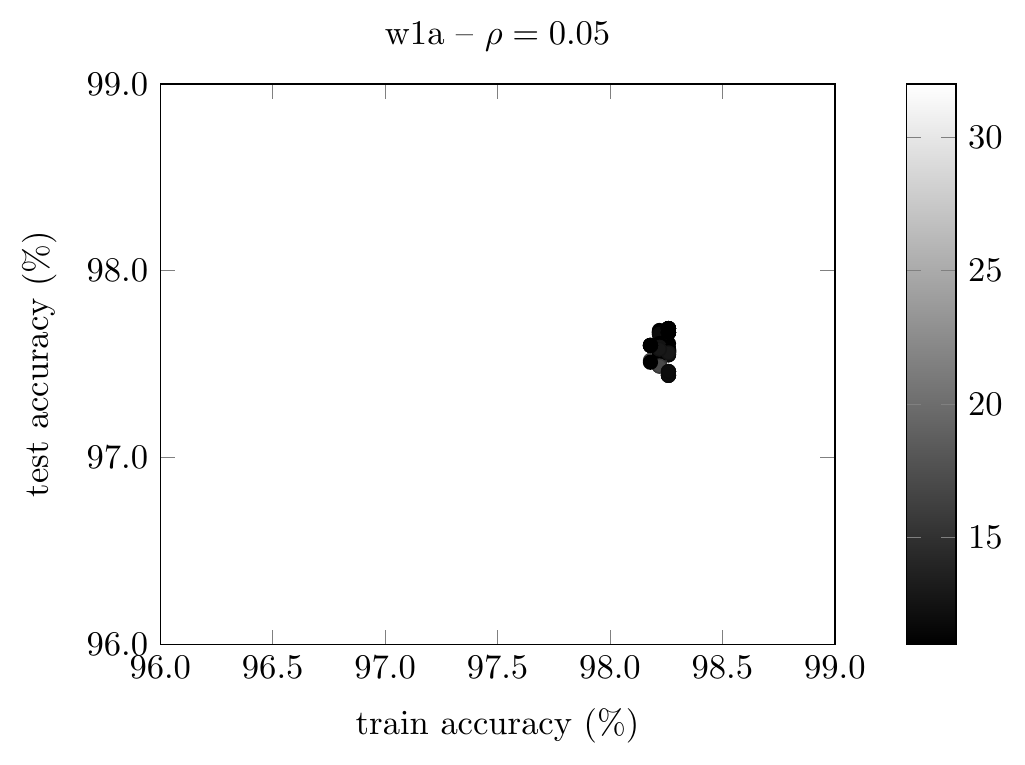}\hfill
  \includegraphics[width=0.33\textwidth]{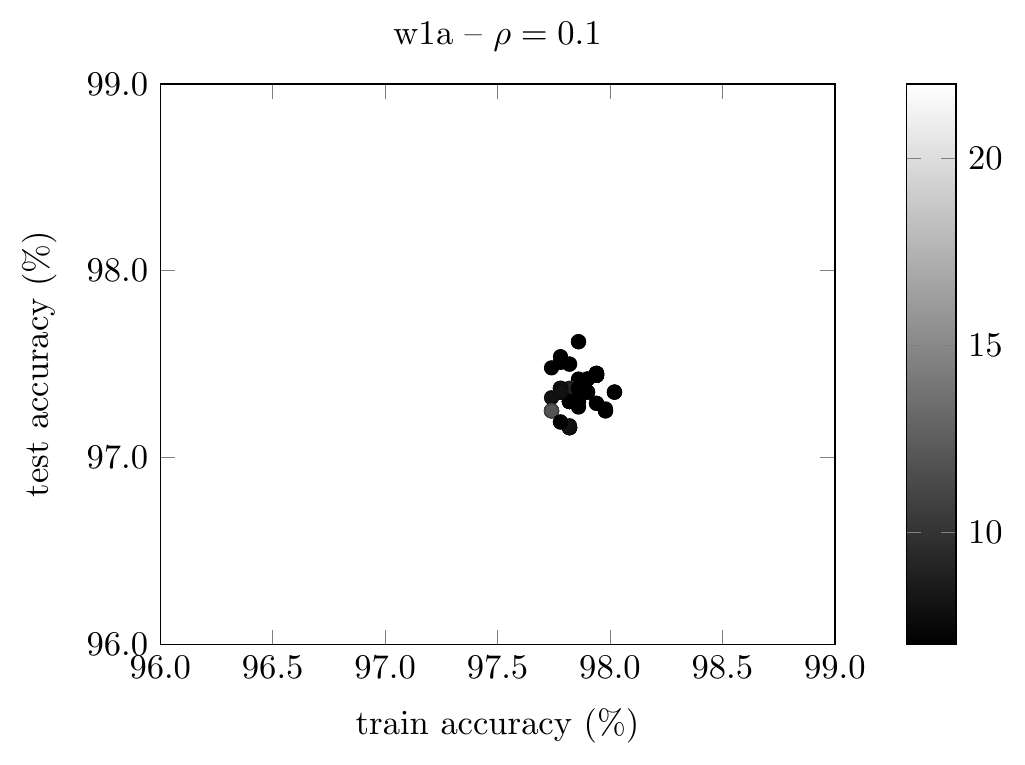}
  \caption{\label{fig:tvt_w1a} \tvtText{w1a}}
\end{figure}

\section{Concluding Remarks}

In this paper, we have first reproduced the observation that boosting based on
column generation, i.e., LP- and IP-boosting, avoids the bad performance on
the well-known hard classes from the literature. More importantly, we have
shown that IP-boosting improves upon LP-boosting and Ada\-Boost on the LIBSVM
instances on which a consistent improvement even by a few percent is not
easy. The price to pay is that the running time with the current
implementation is much longer. Nevertheless, the results are promising, so
it can make sense to tune the performance, e.g., by solving the
intermediate LPs only approximately and deriving tailored heuristics that
generate very good primal solutions, see \cite{BLW08} and
\cite{BLRSW13}, respectively, for examples for column generation in
public transport optimization.

Moreover, our method has a parameter that needs to be tuned,
namely the margin bound \(\rho\). It shares this property with LP-boosting,
where one either needs to set \(\rho\) or a corresponding objective weight.
AdaBoost, however, depends on the number of iterations which also has to
be adjusted to the instance set. We plan to investigate methods based on
the approach in the current paper that avoid the dependence on a parameter.

In conclusion, our approach is suited very well to an offline setting in
which training may take time and where even a small improvement is
beneficial or when convex boosters behave very badly. Moreover, it can
serve as a tool to investigate the general performance of such methods.

% \section{Acknowledgements} We would like to thank Steve Aoki for
% providing inspiration in the late stages of this work.

% \clearpage

\begin{small}

\bibliographystyle{abbrv}
\bibliography{bibliography}

\end{small}

\clearpage

\appendix
\onecolumn

\begin{center}
{\LARGE \ \\ \bigskip Supplementary Material}
\end{center}
\bigskip

\section{Detailed Computational Results}

In the following tables, we report detailed computational results for
our tests. We report problem size statistics in
Table~\ref{tab:instanceStats} and running time statistics in
Table~\ref{tab:runtimeStats}.

For \(\rho = 0.1, 0.075, 0.05, 0.025, 0.01\), we present train
results in Tables~\ref{tab:libsvm:train:0_1}, \ref{tab:libsvm:train:0_075},
\ref{tab:libsvm:train:0_05}, \ref{tab:libsvm:train:0_01} and test results
in Tables~\ref{tab:libsvm:test:0_1}, \ref{tab:libsvm:test:0_075},
\ref{tab:libsvm:test:0_05}, \ref{tab:libsvm:test:0_01}.

\include{tables/table-libsvm-sizes}
\include{tables/table-libsvm-runtimes2}

\include{tables/table-libsvm-train-0_1}
\include{tables/table-libsvm-test-0_1}

\include{tables/table-libsvm-train-0_075}
\include{tables/table-libsvm-test-0_075}

\include{tables/table-libsvm-train-0_05}
\include{tables/table-libsvm-test-0_05}

\include{tables/table-libsvm-train-0_025}
\include{tables/table-libsvm-test-0_025}

\include{tables/table-libsvm-train-0_01}
\include{tables/table-libsvm-test-0_01}

% ----------------------------------------------------------------
\section{Additional Computational tests}
\label{sec:computational-tests}

\subsection{Generalization of IPBoost: train vs. test error}
\label{sec:gener-ipbo-train}

\begin{figure}[h]
  \centering
  \includegraphics[width=0.32\textwidth]{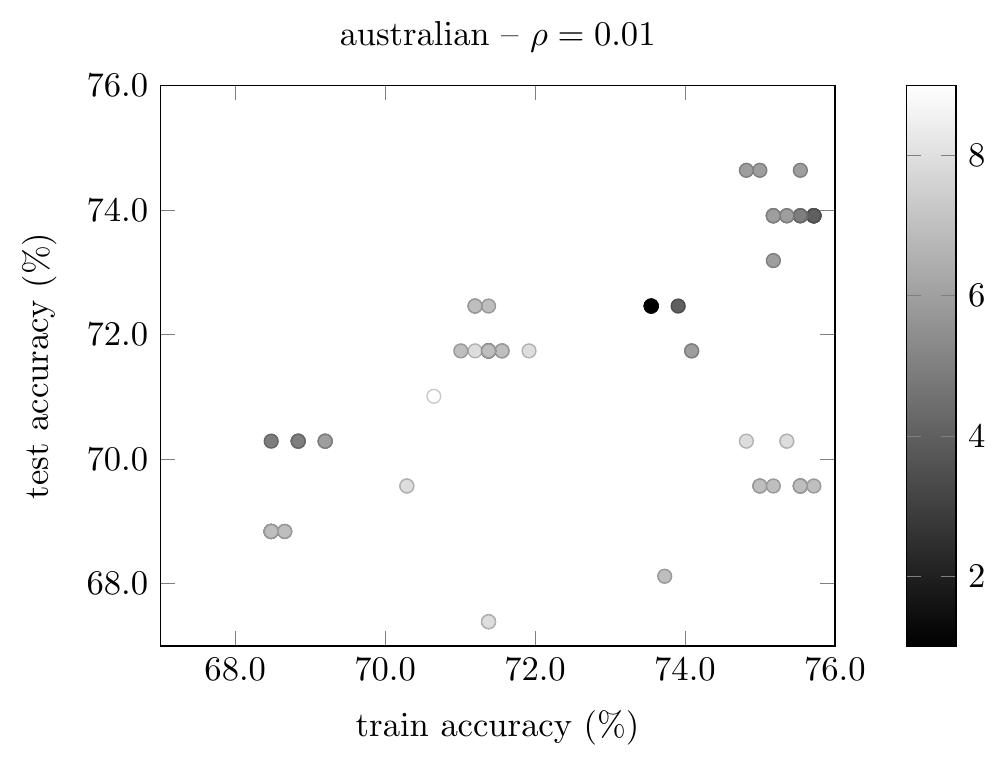}\hfill
  \includegraphics[width=0.32\textwidth]{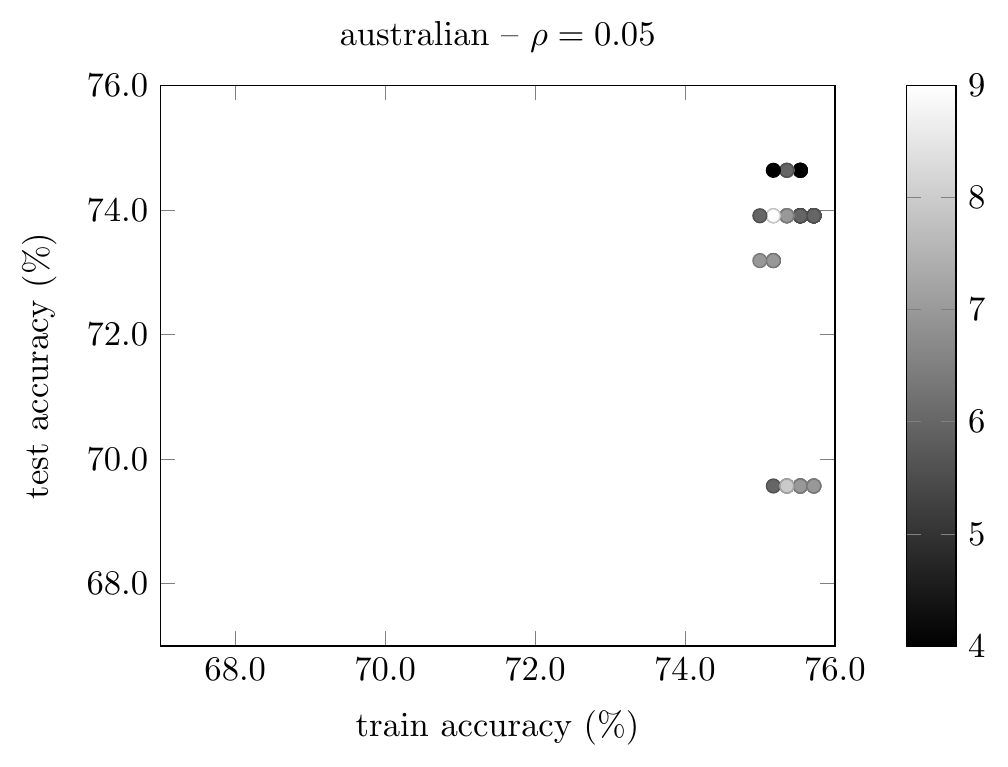}\hfill
  \includegraphics[width=0.32\textwidth]{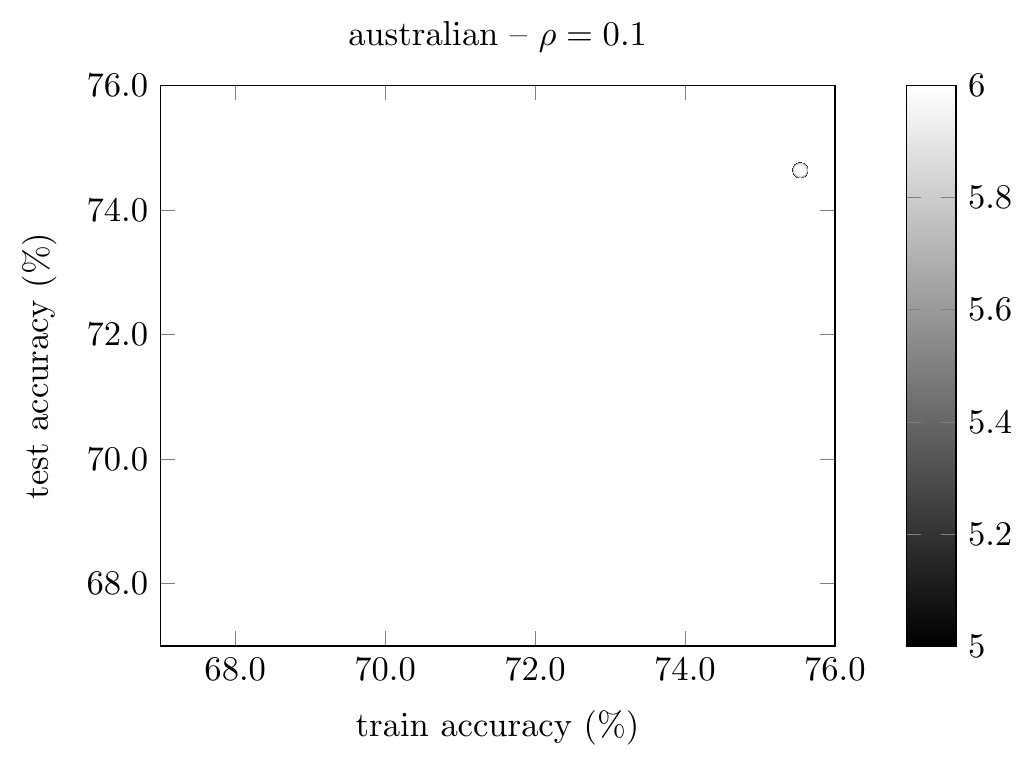}
  \caption{\tvtText{australian}}
\end{figure}

\begin{figure}[h]
  \centering
  \includegraphics[width=0.32\textwidth]{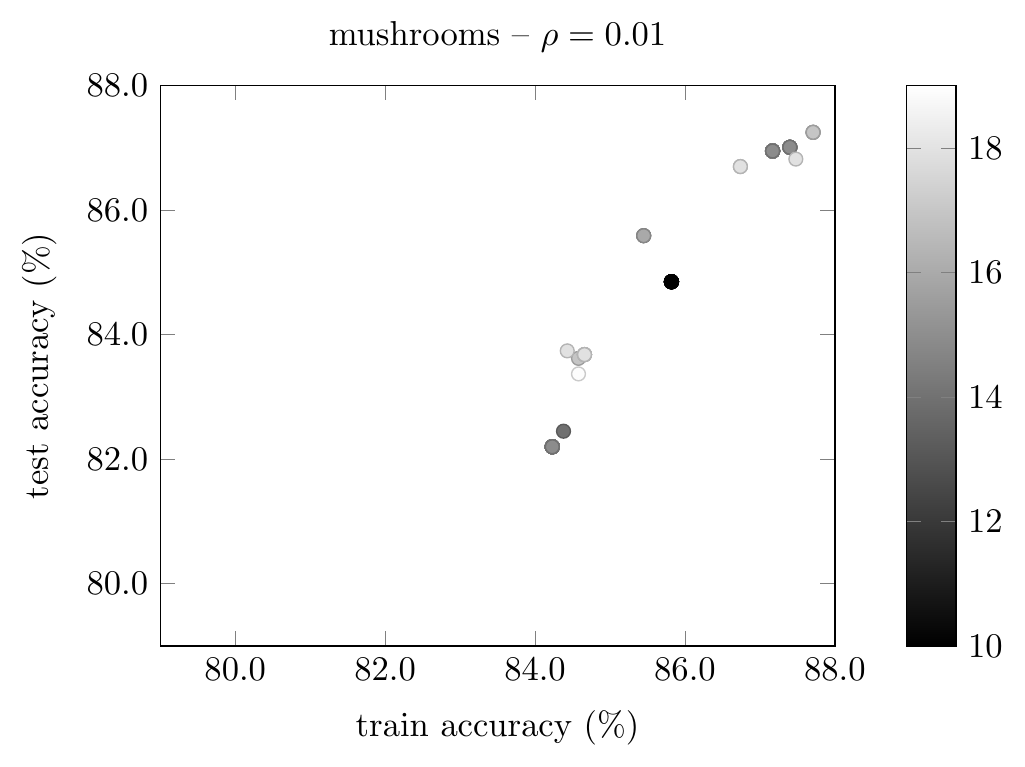}\hfill
  \includegraphics[width=0.32\textwidth]{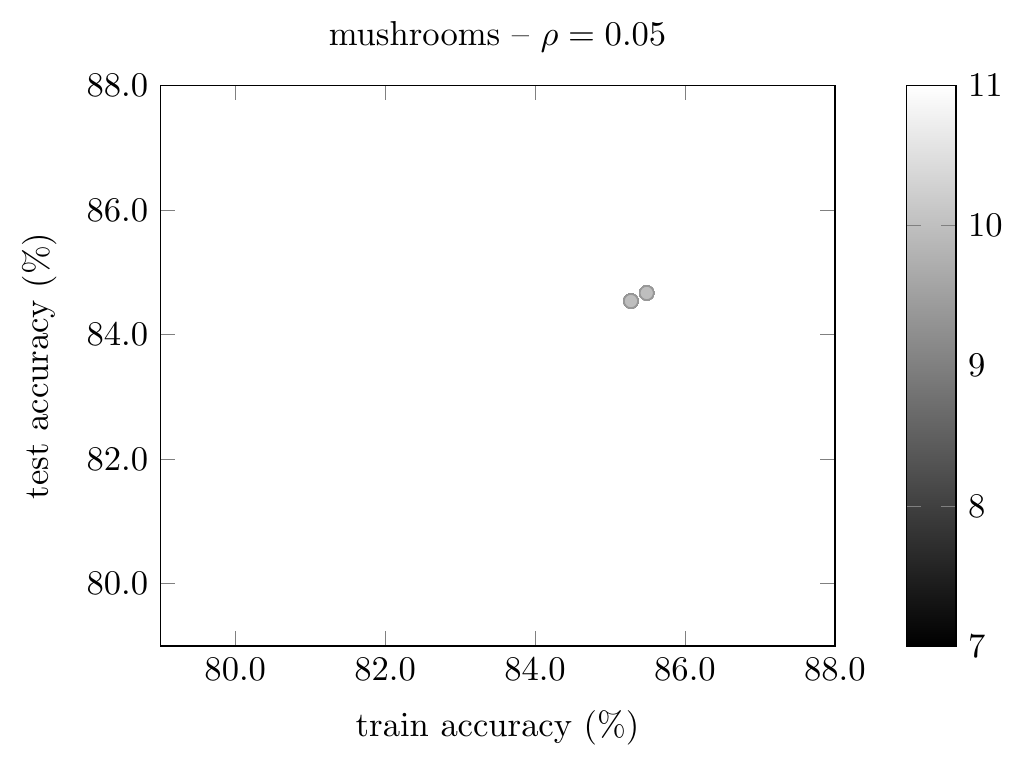}\hfill
  \includegraphics[width=0.32\textwidth]{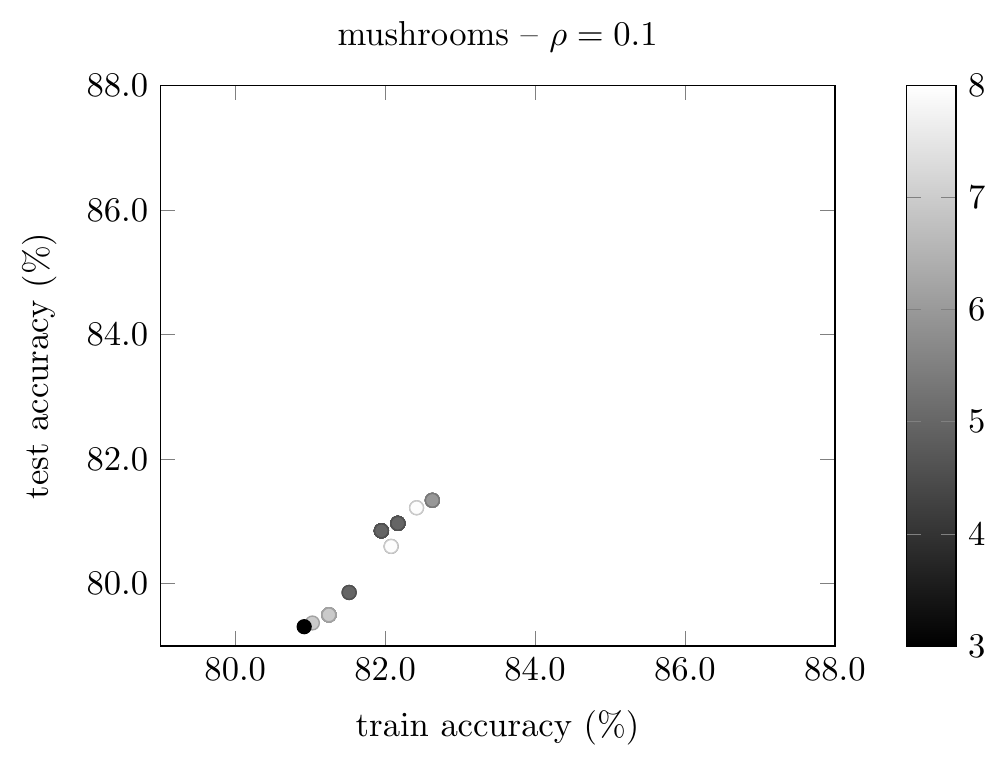}
  \caption{\tvtText{mushrooms}}
\end{figure}

\end{document}

%% file: tables/table-hard2-test.tex
% generated by paper_generate_hard2_table.sh
\begin{table*}
   \small
   \caption{Averages of the \emph{test} accuracies for \emph{hard instances}. The
	 table shows the accuracies and standard deviations as well as the number of learners \(L\) for three algorithms using \(\rho = 0.05\) for 10 different seeds; best solutions are marked with *; using
	 \emph{\(\pm 1\) values for prediction and voting}.}
   \label{tab:hard2}
   \begin{tabular*}{\textwidth}{@{}rr@{\extracolsep{\fill}}r@{\extracolsep{0ex}\quad}r@{\;}>{$}c<{$}@{\;}r@{\quad}r@{\extracolsep{\fill}}r@{\extracolsep{0ex}\quad}r@{\;}>{$}c<{$}@{\;}r@{\quad}r@{\extracolsep{\fill}}r@{\extracolsep{0ex}\quad}r@{\;}>{$}c<{$}@{}r@{\quad}r@{}}\toprule
                       &  &  \multicolumn{5}{c}{IPBoost}                     &  \multicolumn{5}{c}{LPBoost}                       & \multicolumn{5}{c}{AdaBoost}  \\
         \(N\) &\(\gamma\)&   & \multicolumn{3}{c}{score} & \(L\)            &   & \multicolumn{3}{c}{score} & \(L\)              &   & \multicolumn{3}{c}{score} & \(L\)  \\ \midrule
               2000 & 0.1 & * & \num{69.05} & \pm & \num{ 2.54} & \num{ 4.8} &   & \num{  66.22} & \pm & \num{ 1.73} & \num{ 2.0} &   & \num{  58.58} & \pm & \num{ 2.77} & \num{20.9} \\
               4000 & 0.1 & * & \num{68.61} & \pm & \num{ 1.50} & \num{ 4.6} &   & \num{  65.23} & \pm & \num{ 1.95} & \num{ 2.0} &   & \num{  55.45} & \pm & \num{ 2.99} & \num{20.9} \\
               8000 & 0.1 & * & \num{67.26} & \pm & \num{ 1.62} & \num{ 3.6} &   & \num{  64.58} & \pm & \num{ 1.05} & \num{ 2.0} &   & \num{  53.24} & \pm & \num{ 1.68} & \num{20.9} \\
              16000 & 0.1 & * & \num{67.50} & \pm & \num{ 1.48} & \num{ 3.3} &   & \num{  64.73} & \pm & \num{ 0.80} & \num{ 2.0} &   & \num{  51.85} & \pm & \num{ 0.80} & \num{21.0} \\
              32000 & 0.1 & * & \num{67.36} & \pm & \num{ 1.55} & \num{ 2.6} &   & \num{  65.18} & \pm & \num{ 0.55} & \num{ 2.0} &   & \num{  51.22} & \pm & \num{ 0.73} & \num{20.9} \\
              64000 & 0.1 & * & \num{66.65} & \pm & \num{ 1.04} & \num{ 2.5} &   & \num{  65.17} & \pm & \num{ 0.35} & \num{ 2.0} &   & \num{  50.48} & \pm & \num{ 0.49} & \num{20.9} \\
             2000 & 0.075 & * & \num{71.30} & \pm & \num{ 2.06} & \num{ 4.6} &   & \num{  66.55} & \pm & \num{ 1.89} & \num{ 2.0} &   & \num{  57.95} & \pm & \num{ 2.83} & \num{21.1} \\
             4000 & 0.075 & * & \num{70.20} & \pm & \num{ 1.69} & \num{ 4.1} &   & \num{  66.54} & \pm & \num{ 1.58} & \num{ 2.0} &   & \num{  55.27} & \pm & \num{ 2.77} & \num{21.0} \\
             8000 & 0.075 & * & \num{68.41} & \pm & \num{ 1.73} & \num{ 3.8} &   & \num{  65.38} & \pm & \num{ 1.05} & \num{ 2.0} &   & \num{  53.14} & \pm & \num{ 1.51} & \num{21.0} \\
            16000 & 0.075 & * & \num{68.10} & \pm & \num{ 2.18} & \num{ 2.9} &   & \num{  65.63} & \pm & \num{ 0.81} & \num{ 2.0} &   & \num{  51.73} & \pm & \num{ 0.67} & \num{21.0} \\
            32000 & 0.075 & * & \num{68.06} & \pm & \num{ 1.47} & \num{ 2.6} &   & \num{  66.17} & \pm & \num{ 0.62} & \num{ 2.0} &   & \num{  51.12} & \pm & \num{ 0.61} & \num{20.9} \\
            64000 & 0.075 & * & \num{67.92} & \pm & \num{ 1.05} & \num{ 2.4} &   & \num{  66.12} & \pm & \num{ 0.33} & \num{ 2.0} &   & \num{  50.35} & \pm & \num{ 0.47} & \num{21.0} \\
              2000 & 0.05 & * & \num{72.20} & \pm & \num{ 1.92} & \num{ 5.1} &   & \num{  67.05} & \pm & \num{ 1.71} & \num{ 2.0} &   & \num{  57.50} & \pm & \num{ 2.51} & \num{21.0} \\
              4000 & 0.05 & * & \num{71.74} & \pm & \num{ 1.59} & \num{ 4.9} &   & \num{  67.27} & \pm & \num{ 1.69} & \num{ 2.0} &   & \num{  54.75} & \pm & \num{ 2.47} & \num{20.9} \\
              8000 & 0.05 & * & \num{70.09} & \pm & \num{ 1.96} & \num{ 3.4} &   & \num{  66.19} & \pm & \num{ 1.22} & \num{ 2.0} &   & \num{  53.01} & \pm & \num{ 1.40} & \num{21.0} \\
             16000 & 0.05 & * & \num{70.05} & \pm & \num{ 1.57} & \num{ 3.3} &   & \num{  66.82} & \pm & \num{ 0.81} & \num{ 2.0} &   & \num{  51.75} & \pm & \num{ 0.85} & \num{21.0} \\
             32000 & 0.05 & * & \num{69.25} & \pm & \num{ 1.86} & \num{ 2.4} &   & \num{  67.30} & \pm & \num{ 0.54} & \num{ 2.0} &   & \num{  51.15} & \pm & \num{ 0.65} & \num{21.0} \\
             64000 & 0.05 & * & \num{68.83} & \pm & \num{ 1.44} & \num{ 2.3} &   & \num{  67.06} & \pm & \num{ 0.37} & \num{ 2.0} &   & \num{  50.35} & \pm & \num{ 0.54} & \num{21.0} \\
\midrule

\multicolumn{2}{@{}l}{averages:} & \num{18} & \num{69.03} & \pm & \num{ 1.68} & \num{ 3.5} & \num{ 0} & \num{66.07} & \pm & \num{ 1.06} & \num{ 2.0} & \num{ 0} & \num{53.27} & \pm & \num{ 1.49} & \num{21.0}  \\
      \bottomrule
   \end{tabular*}
\end{table*}

%% file: tables/table-libsvm-overview2.tex
% generated by paper_generate_overview_table.sh
\begin{table*}[t]
   \small
   \caption{Aggregated results for LIBSVM: Average test/train accuracies and standard deviations (STD) for three algorithms over 
            10 different seeds, using class probabilities for prediction and voting; we considered 40 instances 
            as outlined in Section~\ref{sec:comp-results}. Column ``\#~best'' represents the number of instances 
            on which the corresponding algorithm performed best (ties possible). Column ``ER'' gives the error rate,
            i.e., \(1/(1-a)\) for the average accuracy~\(a\).}
   \label{tab:aggregated}
   \begin{tabular*}{\textwidth}{@{}l@{\quad\extracolsep{\fill}}rrrrrrrrrrrrr@{}}\toprule
      && \multicolumn{4}{c}{IPBoost} & \multicolumn{4}{c}{LPBoost} & \multicolumn{4}{c}{AdaBoost}\\
         \cmidrule(lr){3-6}\cmidrule(lr){7-10}\cmidrule(l){11-14}
      type & \(\rho\)   & \# best & acc. & STD & ER & \# best & acc. & STD & ER & \# best & acc. & STD & ER \\
      \midrule
 test & \num{  0.1} & \num{24} & \num{80.70} & \num{ 4.08} & \num{ 5.18} & \num{ 7} & \num{80.16} & \num{ 3.93} & \num{ 5.04} & \num{ 9} & \num{79.79} & \num{ 3.82} & \num{ 4.95}\\
 test & \num{0.075} & \num{25} & \num{80.77} & \num{ 3.89} & \num{ 5.20} & \num{ 6} & \num{80.21} & \num{ 3.89} & \num{ 5.05} & \num{ 9} & \num{79.79} & \num{ 3.82} & \num{ 4.95}\\
 test & \num{ 0.05} & \num{26} & \num{80.78} & \num{ 4.07} & \num{ 5.20} & \num{ 7} & \num{80.31} & \num{ 3.73} & \num{ 5.08} & \num{ 8} & \num{79.79} & \num{ 3.82} & \num{ 4.95}\\
 test & \num{0.025} & \num{25} & \num{80.63} & \num{ 3.94} & \num{ 5.16} & \num{ 7} & \num{80.21} & \num{ 3.91} & \num{ 5.05} & \num{ 8} & \num{79.79} & \num{ 3.82} & \num{ 4.95}\\
 test & \num{ 0.01} & \num{26} & \num{80.59} & \num{ 3.91} & \num{ 5.15} & \num{ 7} & \num{79.80} & \num{ 3.67} & \num{ 4.95} & \num{ 7} & \num{79.79} & \num{ 3.82} & \num{ 4.95}\\
      \midrule
train & \num{  0.1} & \num{25} & \num{83.52} & \num{ 2.51} & \num{ 6.07} & \num{ 1} & \num{82.29} & \num{ 2.61} & \num{ 5.65} & \num{15} & \num{84.36} & \num{ 1.99} & \num{ 6.40}\\
train & \num{0.075} & \num{24} & \num{83.94} & \num{ 2.38} & \num{ 6.23} & \num{ 1} & \num{82.52} & \num{ 2.54} & \num{ 5.72} & \num{15} & \num{84.36} & \num{ 1.99} & \num{ 6.40}\\
train & \num{ 0.05} & \num{26} & \num{84.34} & \num{ 2.43} & \num{ 6.38} & \num{ 1} & \num{82.90} & \num{ 2.40} & \num{ 5.85} & \num{14} & \num{84.36} & \num{ 1.99} & \num{ 6.40}\\
train & \num{0.025} & \num{29} & \num{84.97} & \num{ 2.44} & \num{ 6.65} & \num{ 1} & \num{83.39} & \num{ 2.43} & \num{ 6.02} & \num{10} & \num{84.36} & \num{ 1.99} & \num{ 6.40}\\
train & \num{ 0.01} & \num{31} & \num{85.69} & \num{ 2.48} & \num{ 6.99} & \num{ 3} & \num{84.20} & \num{ 2.26} & \num{ 6.33} & \num{ 6} & \num{84.36} & \num{ 1.99} & \num{ 6.40}\\
      \bottomrule
   \end{tabular*}
\end{table*}

%% file: tables/table-libsvm-sizes.tex
\begin{table}
  \small
  \caption{Statistics on LIBSVM instances.}
  \label{tab:instanceStats}
\begin{tabular*}{\textwidth}{@{}l@{\extracolsep{\fill}}rrrrrrrr@{}}\toprule
    & \multicolumn{4}{c}{train set} & \multicolumn{4}{c}{test set} \\ \cmidrule(r){2-5}\cmidrule(l){6-9}
                   name & \(d\) & \(N\) & class \(-1\) & class \(1\) & \(d\) & \(N\) & class \(-1\) & class \(1\)\\
                   \midrule
                      a1a & \num{   119} & \num{  1605} & \num{  1210} & \num{   395} &           -- &           -- &           -- &           -- \\
                      a2a & \num{   119} & \num{  2265} & \num{  1693} & \num{   572} &           -- &           -- &           -- &           -- \\
                      a3a & \num{   122} & \num{  3185} & \num{  2412} & \num{   773} &           -- &           -- &           -- &           -- \\
                      a4a & \num{   122} & \num{  4781} & \num{  3593} & \num{  1188} &           -- &           -- &           -- &           -- \\
                      a5a & \num{   122} & \num{  6414} & \num{  4845} & \num{  1569} &           -- &           -- &           -- &           -- \\
                      a6a & \num{   122} & \num{ 11220} & \num{  8528} & \num{  2692} &           -- &           -- &           -- &           -- \\
\bottomrule
  \end{tabular*}
\end{table}

%%% Local Variables:
%%% mode: latex
%%% TeX-master: "ipBoost_NIPS"
%%% End:

%% file: tables/table-libsvm-runtimes2.tex
% generated by paper_generate_pred_tables.sh
\begin{table}
   \small
   \caption{Statistics for LIBSVM on average run times (in seconds) of different algorithms with $\rho = 0.05$ for 10 different seeds;
            ``\#~optimal'' gives the number of instances solved to optimality, ``\#~time out'' the number of instances that ran into the
            time limit of 1 hour, and ``best sol.\ time'' the average time at which the best solution was found.}
   \label{tab:runtimeStats}
   \begin{tabular*}{\textwidth}{@{}l@{\extracolsep{\fill}}rrrrrrr@{}}\toprule
      & \multicolumn{4}{c}{IPBoost} & &  \\ \cmidrule(r){2-6}
      name & total &          &            &             & best sol.\ & LPBoost & AdaBoost \\
      name &  time & \# nodes & \# optimal & \# time out &       time &    time & time \\ \midrule
                      a1a & \num{  97.26} & \num{ 6669.9} &  0 &  0 & \num{  26.08} & \num{   0.15} & \num{   0.40} \\
                      a2a & \num{ 257.97} & \num{ 8186.6} &  0 &  0 & \num{ 128.55} & \num{   0.26} & \num{   0.59} \\
                      a3a & \num{ 464.05} & \num{ 9096.6} &  0 &  0 & \num{ 230.18} & \num{   0.37} & \num{   0.59} \\
                      a4a & \num{ 770.89} & \num{ 9769.9} &  0 &  0 & \num{ 381.76} & \num{   0.73} & \num{   0.83} \\
                      a5a & \num{ 904.72} & \num{10186.5} &  0 &  0 & \num{ 465.58} & \num{   1.11} & \num{   1.16} \\
\bottomrule
   \end{tabular*}
\end{table}

%% file: tables/table-libsvm-train-0_1.tex
% gnerated by paper_generate_pred_tables.sh
\begin{table}
   \small
   \caption{Averages of the \emph{train} accuracies and standard deviations for three algorithms
            with $\rho = 0.1$ for 10 different seeds on LIBSVM; best solutions are marked with *; using
            class probabilities for prediction and voting.}
   \label{tab:libsvm:train:0_1}
   \begin{tabular*}{\textwidth}{@{}l@{\extracolsep{\fill}}r@{\extracolsep{0ex}\quad}r@{\;}>{$}c<{$}@{\;}r@{\quad}r@{\extracolsep{\fill}}r@{\extracolsep{0ex}\quad}r@{\;}>{$}c<{$}@{\;}r@{\quad}r@{\extracolsep{\fill}}r@{\extracolsep{0ex}\quad}r@{\;}>{$}c<{$}@{}r@{\quad}r@{}}\toprule
                          &  \multicolumn{5}{c}{IPBoost}                     &  \multicolumn{5}{c}{LPBoost}                       & \multicolumn{5}{c}{AdaBoost}  \\
                     name &   & \multicolumn{3}{c}{score} & \(L\)            &   & \multicolumn{3}{c}{score} & \(L\)              &   & \multicolumn{3}{c}{score} & \(L\)  \\ \midrule
                      a1a & * & \num{79.23} & \pm & \num{ 1.93} & \num{ 4.8} &   & \num{  77.41} & \pm & \num{ 1.23} & \num{ 4.1} &   & \num{  77.31} & \pm & \num{ 1.00} & \num{ 9.3} \\
                      a2a & * & \num{78.20} & \pm & \num{ 1.82} & \num{ 4.1} &   & \num{  77.00} & \pm & \num{ 1.79} & \num{ 5.7} &   & \num{  76.42} & \pm & \num{ 0.94} & \num{ 9.6} \\
                      a3a & * & \num{77.42} & \pm & \num{ 1.02} & \num{ 4.2} &   & \num{  76.41} & \pm & \num{ 0.80} & \num{ 4.6} &   & \num{  77.00} & \pm & \num{ 1.10} & \num{10.9} \\
                      a4a & * & \num{77.02} & \pm & \num{ 1.03} & \num{ 4.0} &   & \num{  75.89} & \pm & \num{ 0.82} & \num{ 4.6} &   & \num{  76.84} & \pm & \num{ 1.07} & \num{11.0} \\
                      a5a & * & \num{77.75} & \pm & \num{ 0.97} & \num{ 4.3} &   & \num{  76.39} & \pm & \num{ 0.87} & \num{ 4.0} &   & \num{  77.19} & \pm & \num{ 0.99} & \num{11.0} \\
                      a6a & * & \num{77.45} & \pm & \num{ 0.79} & \num{ 3.9} &   & \num{  76.99} & \pm & \num{ 1.02} & \num{ 4.3} &   & \num{  77.30} & \pm & \num{ 1.03} & \num{11.0} \\
\bottomrule
   \end{tabular*}
\end{table}

%% file: tables/table-libsvm-test-0_1.tex
% gnerated by paper_generate_pred_tables.sh
\begin{table}
   \small
   \caption{Averages of the \emph{test} accuracies and standard deviations for three algorithms
            with $\rho = 0.1$ for 10 different seeds on LIBSVM; best solutions are marked with *; using
            class probabilities for prediction and voting.}
   \label{tab:libsvm:test:0_1}
   \begin{tabular*}{\textwidth}{@{}l@{\extracolsep{\fill}}r@{\extracolsep{0ex}\quad}r@{\;}>{$}c<{$}@{\;}r@{\quad}r@{\extracolsep{\fill}}r@{\extracolsep{0ex}\quad}r@{\;}>{$}c<{$}@{\;}r@{\quad}r@{\extracolsep{\fill}}r@{\extracolsep{0ex}\quad}r@{\;}>{$}c<{$}@{}r@{\quad}r@{}}\toprule
                          &  \multicolumn{5}{c}{IPBoost}                     &  \multicolumn{5}{c}{LPBoost}                       & \multicolumn{5}{c}{AdaBoost}  \\
                     name &   & \multicolumn{3}{c}{score} & \(L\)            &   & \multicolumn{3}{c}{score} & \(L\)              &   & \multicolumn{3}{c}{score} & \(L\)  \\ \midrule
                      a1a & * & \num{78.32} & \pm & \num{ 3.25} & \num{ 4.8} &   & \num{  76.70} & \pm & \num{ 2.94} & \num{ 4.1} &   & \num{  76.04} & \pm & \num{ 3.30} & \num{ 9.3} \\
                      a2a & * & \num{77.88} & \pm & \num{ 2.23} & \num{ 4.1} &   & \num{  76.51} & \pm & \num{ 2.30} & \num{ 5.7} &   & \num{  75.67} & \pm & \num{ 1.57} & \num{ 9.6} \\
                      a3a & * & \num{77.47} & \pm & \num{ 1.39} & \num{ 4.2} &   & \num{  75.64} & \pm & \num{ 1.62} & \num{ 4.6} &   & \num{  76.84} & \pm & \num{ 1.03} & \num{10.9} \\
                      a4a & * & \num{76.44} & \pm & \num{ 1.93} & \num{ 4.0} &   & \num{  75.23} & \pm & \num{ 1.75} & \num{ 4.6} &   & \num{  76.18} & \pm & \num{ 1.72} & \num{11.0} \\
                      a5a & * & \num{77.18} & \pm & \num{ 1.51} & \num{ 4.3} &   & \num{  75.70} & \pm & \num{ 1.85} & \num{ 4.0} &   & \num{  76.72} & \pm & \num{ 1.77} & \num{11.0} \\
                      a6a &   & \num{77.77} & \pm & \num{ 1.22} & \num{ 3.9} &   & \num{  77.58} & \pm & \num{ 1.35} & \num{ 4.3} & * & \num{  77.84} & \pm & \num{ 1.31} & \num{11.0} \\
\bottomrule
   \end{tabular*}
\end{table}

%% file: tables/table-libsvm-train-0_075.tex
% gnerated by paper_generate_pred_tables.sh
\begin{table}
   \small
   \caption{Averages of the \emph{train} accuracies and standard deviations for three algorithms
            with $\rho = 0.075$ for 10 different seeds on LIBSVM; best solutions are marked with *; using
            class probabilities for prediction and voting.}
   \label{tab:libsvm:train:0_075}
   \begin{tabular*}{\textwidth}{@{}l@{\extracolsep{\fill}}r@{\extracolsep{0ex}\quad}r@{\;}>{$}c<{$}@{\;}r@{\quad}r@{\extracolsep{\fill}}r@{\extracolsep{0ex}\quad}r@{\;}>{$}c<{$}@{\;}r@{\quad}r@{\extracolsep{\fill}}r@{\extracolsep{0ex}\quad}r@{\;}>{$}c<{$}@{}r@{\quad}r@{}}\toprule
                          &  \multicolumn{5}{c}{IPBoost}                     &  \multicolumn{5}{c}{LPBoost}                       & \multicolumn{5}{c}{AdaBoost}  \\
                     name &   & \multicolumn{3}{c}{score} & \(L\)            &   & \multicolumn{3}{c}{score} & \(L\)              &   & \multicolumn{3}{c}{score} & \(L\)  \\ \midrule
                      a1a & * & \num{79.35} & \pm & \num{ 2.14} & \num{ 5.5} &   & \num{  77.37} & \pm & \num{ 1.20} & \num{ 4.7} &   & \num{  77.31} & \pm & \num{ 1.00} & \num{ 9.3} \\
                      a2a & * & \num{78.68} & \pm & \num{ 1.60} & \num{ 5.4} &   & \num{  77.25} & \pm & \num{ 1.88} & \num{ 6.3} &   & \num{  76.42} & \pm & \num{ 0.94} & \num{ 9.6} \\
                      a3a & * & \num{78.05} & \pm & \num{ 1.11} & \num{ 6.0} &   & \num{  76.43} & \pm & \num{ 0.82} & \num{ 5.0} &   & \num{  77.00} & \pm & \num{ 1.10} & \num{10.9} \\
                      a4a & * & \num{77.73} & \pm & \num{ 0.95} & \num{ 6.0} &   & \num{  75.90} & \pm & \num{ 0.82} & \num{ 4.5} &   & \num{  76.84} & \pm & \num{ 1.07} & \num{11.0} \\
                      a5a & * & \num{78.18} & \pm & \num{ 1.28} & \num{ 5.3} &   & \num{  76.40} & \pm & \num{ 0.87} & \num{ 4.4} &   & \num{  77.19} & \pm & \num{ 0.99} & \num{11.0} \\
                      a6a & * & \num{77.68} & \pm & \num{ 0.99} & \num{ 4.6} &   & \num{  76.73} & \pm & \num{ 0.95} & \num{ 3.9} &   & \num{  77.30} & \pm & \num{ 1.03} & \num{11.0} \\
\bottomrule
   \end{tabular*}
\end{table}

%% file: tables/table-libsvm-test-0_075.tex
% gnerated by paper_generate_pred_tables.sh
\begin{table}
   \small
   \caption{Averages of the \emph{test} accuracies and standard deviations for three algorithms
            with $\rho = 0.075$ for 10 different seeds on LIBSVM; best solutions are marked with *; using
            class probabilities for prediction and voting.}
   \label{tab:libsvm:test:0_075}
   \begin{tabular*}{\textwidth}{@{}l@{\extracolsep{\fill}}r@{\extracolsep{0ex}\quad}r@{\;}>{$}c<{$}@{\;}r@{\quad}r@{\extracolsep{\fill}}r@{\extracolsep{0ex}\quad}r@{\;}>{$}c<{$}@{\;}r@{\quad}r@{\extracolsep{\fill}}r@{\extracolsep{0ex}\quad}r@{\;}>{$}c<{$}@{}r@{\quad}r@{}}\toprule
                          &  \multicolumn{5}{c}{IPBoost}                     &  \multicolumn{5}{c}{LPBoost}                       & \multicolumn{5}{c}{AdaBoost}  \\
                     name &   & \multicolumn{3}{c}{score} & \(L\)            &   & \multicolumn{3}{c}{score} & \(L\)              &   & \multicolumn{3}{c}{score} & \(L\)  \\ \midrule
                      a1a & * & \num{78.60} & \pm & \num{ 3.49} & \num{ 5.5} &   & \num{  76.73} & \pm & \num{ 2.94} & \num{ 4.7} &   & \num{  76.04} & \pm & \num{ 3.30} & \num{ 9.3} \\
                      a2a & * & \num{78.28} & \pm & \num{ 1.69} & \num{ 5.4} &   & \num{  76.62} & \pm & \num{ 2.35} & \num{ 6.3} &   & \num{  75.67} & \pm & \num{ 1.57} & \num{ 9.6} \\
                      a3a & * & \num{77.60} & \pm & \num{ 1.19} & \num{ 6.0} &   & \num{  75.60} & \pm & \num{ 1.60} & \num{ 5.0} &   & \num{  76.84} & \pm & \num{ 1.03} & \num{10.9} \\
                      a4a & * & \num{76.90} & \pm & \num{ 2.10} & \num{ 6.0} &   & \num{  75.23} & \pm & \num{ 1.75} & \num{ 4.5} &   & \num{  76.18} & \pm & \num{ 1.72} & \num{11.0} \\
                      a5a & * & \num{77.46} & \pm & \num{ 1.71} & \num{ 5.3} &   & \num{  75.73} & \pm & \num{ 1.87} & \num{ 4.4} &   & \num{  76.72} & \pm & \num{ 1.77} & \num{11.0} \\
                      a6a & * & \num{78.08} & \pm & \num{ 1.23} & \num{ 4.6} &   & \num{  77.42} & \pm & \num{ 1.05} & \num{ 3.9} &   & \num{  77.84} & \pm & \num{ 1.31} & \num{11.0} \\
\bottomrule
   \end{tabular*}
\end{table}

%% file: tables/table-libsvm-train-0_05.tex
% gnerated by paper_generate_pred_tables.sh
\begin{table}
   \small
   \caption{Averages of the \emph{train} accuracies and standard deviations for three algorithms
            with $\rho = 0.05$ for 10 different seeds on LIBSVM; best solutions are marked with *; using
            class probabilities for prediction and voting.}
   \label{tab:libsvm:train:0_05}
   \begin{tabular*}{\textwidth}{@{}l@{\extracolsep{\fill}}r@{\extracolsep{0ex}\quad}r@{\;}>{$}c<{$}@{\;}r@{\quad}r@{\extracolsep{\fill}}r@{\extracolsep{0ex}\quad}r@{\;}>{$}c<{$}@{\;}r@{\quad}r@{\extracolsep{\fill}}r@{\extracolsep{0ex}\quad}r@{\;}>{$}c<{$}@{}r@{\quad}r@{}}\toprule
                          &  \multicolumn{5}{c}{IPBoost}                     &  \multicolumn{5}{c}{LPBoost}                       & \multicolumn{5}{c}{AdaBoost}  \\
                     name &   & \multicolumn{3}{c}{score} & \(L\)            &   & \multicolumn{3}{c}{score} & \(L\)              &   & \multicolumn{3}{c}{score} & \(L\)  \\ \midrule
                      a1a & * & \num{79.97} & \pm & \num{ 1.80} & \num{ 7.0} &   & \num{  77.41} & \pm & \num{ 1.23} & \num{ 4.6} &   & \num{  77.31} & \pm & \num{ 1.00} & \num{ 9.3} \\
                      a2a & * & \num{79.06} & \pm & \num{ 1.73} & \num{ 6.2} &   & \num{  76.66} & \pm & \num{ 1.01} & \num{ 6.5} &   & \num{  76.42} & \pm & \num{ 0.94} & \num{ 9.6} \\
                      a3a & * & \num{78.67} & \pm & \num{ 1.08} & \num{ 7.9} &   & \num{  76.37} & \pm & \num{ 0.83} & \num{ 5.0} &   & \num{  77.00} & \pm & \num{ 1.10} & \num{10.9} \\
                      a4a & * & \num{78.24} & \pm & \num{ 1.27} & \num{ 7.8} &   & \num{  75.93} & \pm & \num{ 0.82} & \num{ 4.8} &   & \num{  76.84} & \pm & \num{ 1.07} & \num{11.0} \\
                      a5a & * & \num{78.46} & \pm & \num{ 1.18} & \num{ 7.8} &   & \num{  76.41} & \pm & \num{ 0.87} & \num{ 4.4} &   & \num{  77.19} & \pm & \num{ 0.99} & \num{11.0} \\
                      a6a & * & \num{77.94} & \pm & \num{ 0.84} & \num{ 6.7} &   & \num{  76.94} & \pm & \num{ 0.96} & \num{ 4.9} &   & \num{  77.30} & \pm & \num{ 1.03} & \num{11.0} \\
\bottomrule
   \end{tabular*}
\end{table}

%% file: tables/table-libsvm-test-0_05.tex
% gnerated by paper_generate_pred_tables.sh
\begin{table}
   \small
   \caption{Averages of the \emph{test} accuracies and standard deviations for three algorithms
            with $\rho = 0.05$ for 10 different seeds on LIBSVM; best solutions are marked with *; using
            class probabilities for prediction and voting.}
   \label{tab:libsvm:test:0_05}
   \begin{tabular*}{\textwidth}{@{}l@{\extracolsep{\fill}}r@{\extracolsep{0ex}\quad}r@{\;}>{$}c<{$}@{\;}r@{\quad}r@{\extracolsep{\fill}}r@{\extracolsep{0ex}\quad}r@{\;}>{$}c<{$}@{\;}r@{\quad}r@{\extracolsep{\fill}}r@{\extracolsep{0ex}\quad}r@{\;}>{$}c<{$}@{}r@{\quad}r@{}}\toprule
                          &  \multicolumn{5}{c}{IPBoost}                     &  \multicolumn{5}{c}{LPBoost}                       & \multicolumn{5}{c}{AdaBoost}  \\
                     name &   & \multicolumn{3}{c}{score} & \(L\)            &   & \multicolumn{3}{c}{score} & \(L\)              &   & \multicolumn{3}{c}{score} & \(L\)  \\ \midrule
                      a1a & * & \num{79.10} & \pm & \num{ 3.44} & \num{ 7.0} &   & \num{  76.70} & \pm & \num{ 2.99} & \num{ 4.6} &   & \num{  76.04} & \pm & \num{ 3.30} & \num{ 9.3} \\
                      a2a & * & \num{78.08} & \pm & \num{ 2.22} & \num{ 6.2} &   & \num{  75.83} & \pm & \num{ 2.36} & \num{ 6.5} &   & \num{  75.67} & \pm & \num{ 1.57} & \num{ 9.6} \\
                      a3a & * & \num{78.32} & \pm & \num{ 1.42} & \num{ 7.9} &   & \num{  75.54} & \pm & \num{ 1.46} & \num{ 5.0} &   & \num{  76.84} & \pm & \num{ 1.03} & \num{10.9} \\
                      a4a & * & \num{77.55} & \pm & \num{ 2.21} & \num{ 7.8} &   & \num{  75.26} & \pm & \num{ 1.78} & \num{ 4.8} &   & \num{  76.18} & \pm & \num{ 1.72} & \num{11.0} \\
                      a5a & * & \num{77.39} & \pm & \num{ 1.25} & \num{ 7.8} &   & \num{  75.74} & \pm & \num{ 1.88} & \num{ 4.4} &   & \num{  76.72} & \pm & \num{ 1.77} & \num{11.0} \\
                      a6a & * & \num{78.48} & \pm & \num{ 1.13} & \num{ 6.7} &   & \num{  77.54} & \pm & \num{ 1.27} & \num{ 4.9} &   & \num{  77.84} & \pm & \num{ 1.31} & \num{11.0} \\
\bottomrule
   \end{tabular*}
\end{table}

%% file: tables/table-libsvm-train-0_025.tex
% gnerated by paper_generate_pred_tables.sh
\begin{table}
   \small
   \caption{Averages of the \emph{train} accuracies and standard deviations for three algorithms
            with $\rho = 0.025$ for 10 different seeds on LIBSVM; best solutions are marked with *; using
            class probabilities for prediction and voting.}
   \label{tab:libsvm:train:0_025}
   \begin{tabular*}{\textwidth}{@{}l@{\extracolsep{\fill}}r@{\extracolsep{0ex}\quad}r@{\;}>{$}c<{$}@{\;}r@{\quad}r@{\extracolsep{\fill}}r@{\extracolsep{0ex}\quad}r@{\;}>{$}c<{$}@{\;}r@{\quad}r@{\extracolsep{\fill}}r@{\extracolsep{0ex}\quad}r@{\;}>{$}c<{$}@{}r@{\quad}r@{}}\toprule
                          &  \multicolumn{5}{c}{IPBoost}                     &  \multicolumn{5}{c}{LPBoost}                       & \multicolumn{5}{c}{AdaBoost}  \\
                     name &   & \multicolumn{3}{c}{score} & \(L\)            &   & \multicolumn{3}{c}{score} & \(L\)              &   & \multicolumn{3}{c}{score} & \(L\)  \\ \midrule
                      a1a & * & \num{80.80} & \pm & \num{ 1.83} & \num{10.8} &   & \num{  77.38} & \pm & \num{ 1.21} & \num{ 4.9} &   & \num{  77.31} & \pm & \num{ 1.00} & \num{ 9.3} \\
                      a2a & * & \num{79.71} & \pm & \num{ 1.46} & \num{10.7} &   & \num{  76.72} & \pm & \num{ 1.02} & \num{ 5.7} &   & \num{  76.42} & \pm & \num{ 0.94} & \num{ 9.6} \\
                      a3a & * & \num{78.77} & \pm & \num{ 1.20} & \num{11.9} &   & \num{  76.43} & \pm & \num{ 0.82} & \num{ 4.9} &   & \num{  77.00} & \pm & \num{ 1.10} & \num{10.9} \\
                      a4a & * & \num{78.26} & \pm & \num{ 1.42} & \num{11.4} &   & \num{  75.90} & \pm & \num{ 0.82} & \num{ 4.6} &   & \num{  76.84} & \pm & \num{ 1.07} & \num{11.0} \\
                      a5a & * & \num{78.99} & \pm & \num{ 1.46} & \num{11.6} &   & \num{  76.41} & \pm & \num{ 0.87} & \num{ 4.5} &   & \num{  77.19} & \pm & \num{ 0.99} & \num{11.0} \\
                      a6a & * & \num{78.06} & \pm & \num{ 1.17} & \num{ 8.6} &   & \num{  76.75} & \pm & \num{ 0.96} & \num{ 4.6} &   & \num{  77.30} & \pm & \num{ 1.03} & \num{11.0} \\
\bottomrule
   \end{tabular*}
\end{table}

%% file: tables/table-libsvm-test-0_025.tex
% gnerated by paper_generate_pred_tables.sh
\begin{table}
   \small
   \caption{Averages of the \emph{test} accuracies and standard deviations for three algorithms
            with $\rho = 0.025$ for 10 different seeds on LIBSVM; best solutions are marked with *; using
            class probabilities for prediction and voting.}
   \label{tab:libsvm:test:0_025}
   \begin{tabular*}{\textwidth}{@{}l@{\extracolsep{\fill}}r@{\extracolsep{0ex}\quad}r@{\;}>{$}c<{$}@{\;}r@{\quad}r@{\extracolsep{\fill}}r@{\extracolsep{0ex}\quad}r@{\;}>{$}c<{$}@{\;}r@{\quad}r@{\extracolsep{\fill}}r@{\extracolsep{0ex}\quad}r@{\;}>{$}c<{$}@{}r@{\quad}r@{}}\toprule
                          &  \multicolumn{5}{c}{IPBoost}                     &  \multicolumn{5}{c}{LPBoost}                       & \multicolumn{5}{c}{AdaBoost}  \\
                     name &   & \multicolumn{3}{c}{score} & \(L\)            &   & \multicolumn{3}{c}{score} & \(L\)              &   & \multicolumn{3}{c}{score} & \(L\)  \\ \midrule
                      a1a & * & \num{80.12} & \pm & \num{ 2.80} & \num{10.8} &   & \num{  76.70} & \pm & \num{ 2.99} & \num{ 4.9} &   & \num{  76.04} & \pm & \num{ 3.30} & \num{ 9.3} \\
                      a2a & * & \num{78.98} & \pm & \num{ 1.83} & \num{10.7} &   & \num{  75.96} & \pm & \num{ 2.26} & \num{ 5.7} &   & \num{  75.67} & \pm & \num{ 1.57} & \num{ 9.6} \\
                      a3a & * & \num{78.15} & \pm & \num{ 1.64} & \num{11.9} &   & \num{  75.60} & \pm & \num{ 1.60} & \num{ 4.9} &   & \num{  76.84} & \pm & \num{ 1.03} & \num{10.9} \\
                      a4a & * & \num{76.98} & \pm & \num{ 1.75} & \num{11.4} &   & \num{  75.23} & \pm & \num{ 1.75} & \num{ 4.6} &   & \num{  76.18} & \pm & \num{ 1.72} & \num{11.0} \\
                      a5a & * & \num{77.81} & \pm & \num{ 1.72} & \num{11.6} &   & \num{  75.74} & \pm & \num{ 1.88} & \num{ 4.5} &   & \num{  76.72} & \pm & \num{ 1.77} & \num{11.0} \\
                      a6a & * & \num{78.40} & \pm & \num{ 1.23} & \num{ 8.6} &   & \num{  77.44} & \pm & \num{ 1.04} & \num{ 4.6} &   & \num{  77.84} & \pm & \num{ 1.31} & \num{11.0} \\
\bottomrule
   \end{tabular*}
\end{table}

%% file: tables/table-libsvm-train-0_01.tex
% gnerated by paper_generate_pred_tables.sh
\begin{table}
   \small
   \caption{Averages of the \emph{train} accuracies and standard deviations for three algorithms
            with $\rho = 0.01$ for 10 different seeds on LIBSVM; best solutions are marked with *; using
            class probabilities for prediction and voting.}
   \label{tab:libsvm:train:0_01}
   \begin{tabular*}{\textwidth}{@{}l@{\extracolsep{\fill}}r@{\extracolsep{0ex}\quad}r@{\;}>{$}c<{$}@{\;}r@{\quad}r@{\extracolsep{\fill}}r@{\extracolsep{0ex}\quad}r@{\;}>{$}c<{$}@{\;}r@{\quad}r@{\extracolsep{\fill}}r@{\extracolsep{0ex}\quad}r@{\;}>{$}c<{$}@{}r@{\quad}r@{}}\toprule
                          &  \multicolumn{5}{c}{IPBoost}                     &  \multicolumn{5}{c}{LPBoost}                       & \multicolumn{5}{c}{AdaBoost}  \\
                     name &   & \multicolumn{3}{c}{score} & \(L\)            &   & \multicolumn{3}{c}{score} & \(L\)              &   & \multicolumn{3}{c}{score} & \(L\)  \\ \midrule
                      a1a & * & \num{79.97} & \pm & \num{ 2.23} & \num{10.3} &   & \num{  77.41} & \pm & \num{ 1.23} & \num{ 4.5} &   & \num{  77.31} & \pm & \num{ 1.00} & \num{ 9.3} \\
                      a2a & * & \num{79.18} & \pm & \num{ 1.86} & \num{11.8} &   & \num{  76.74} & \pm & \num{ 0.92} & \num{ 5.4} &   & \num{  76.42} & \pm & \num{ 0.94} & \num{ 9.6} \\
                      a3a & * & \num{78.63} & \pm & \num{ 1.54} & \num{13.9} &   & \num{  76.39} & \pm & \num{ 0.82} & \num{ 4.5} &   & \num{  77.00} & \pm & \num{ 1.10} & \num{10.9} \\
                      a4a & * & \num{78.68} & \pm & \num{ 1.55} & \num{13.7} &   & \num{  75.90} & \pm & \num{ 0.82} & \num{ 4.6} &   & \num{  76.84} & \pm & \num{ 1.07} & \num{11.0} \\
                      a5a & * & \num{79.05} & \pm & \num{ 1.57} & \num{14.7} &   & \num{  76.31} & \pm & \num{ 0.80} & \num{ 4.4} &   & \num{  77.19} & \pm & \num{ 0.99} & \num{11.0} \\
                      a6a & * & \num{78.33} & \pm & \num{ 1.31} & \num{11.5} &   & \num{  76.73} & \pm & \num{ 0.96} & \num{ 4.4} &   & \num{  77.30} & \pm & \num{ 1.03} & \num{11.0} \\
\bottomrule
   \end{tabular*}
\end{table}

%% file: tables/table-libsvm-test-0_01.tex
% gnerated by paper_generate_pred_tables.sh
\begin{table}
   \small
   \caption{Averages of the \emph{test} accuracies and standard deviations for three algorithms
            with $\rho = 0.01$ for 10 different seeds on LIBSVM; best solutions are marked with *; using
            class probabilities for prediction and voting.}
   \label{tab:libsvm:test:0_01}
   \begin{tabular*}{\textwidth}{@{}l@{\extracolsep{\fill}}r@{\extracolsep{0ex}\quad}r@{\;}>{$}c<{$}@{\;}r@{\quad}r@{\extracolsep{\fill}}r@{\extracolsep{0ex}\quad}r@{\;}>{$}c<{$}@{\;}r@{\quad}r@{\extracolsep{\fill}}r@{\extracolsep{0ex}\quad}r@{\;}>{$}c<{$}@{}r@{\quad}r@{}}\toprule
                          &  \multicolumn{5}{c}{IPBoost}                     &  \multicolumn{5}{c}{LPBoost}                       & \multicolumn{5}{c}{AdaBoost}  \\
                     name &   & \multicolumn{3}{c}{score} & \(L\)            &   & \multicolumn{3}{c}{score} & \(L\)              &   & \multicolumn{3}{c}{score} & \(L\)  \\ \midrule
                      a1a & * & \num{78.47} & \pm & \num{ 3.33} & \num{10.3} &   & \num{  76.73} & \pm & \num{ 2.94} & \num{ 4.5} &   & \num{  76.04} & \pm & \num{ 3.30} & \num{ 9.3} \\
                      a2a & * & \num{77.75} & \pm & \num{ 1.97} & \num{11.8} &   & \num{  76.03} & \pm & \num{ 2.36} & \num{ 5.4} &   & \num{  75.67} & \pm & \num{ 1.57} & \num{ 9.6} \\
                      a3a & * & \num{78.02} & \pm & \num{ 1.50} & \num{13.9} &   & \num{  75.59} & \pm & \num{ 1.56} & \num{ 4.5} &   & \num{  76.84} & \pm & \num{ 1.03} & \num{10.9} \\
                      a4a & * & \num{77.63} & \pm & \num{ 2.54} & \num{13.7} &   & \num{  75.23} & \pm & \num{ 1.75} & \num{ 4.6} &   & \num{  76.18} & \pm & \num{ 1.72} & \num{11.0} \\
                      a5a & * & \num{77.96} & \pm & \num{ 1.64} & \num{14.7} &   & \num{  75.59} & \pm & \num{ 1.86} & \num{ 4.4} &   & \num{  76.72} & \pm & \num{ 1.77} & \num{11.0} \\
                      a6a & * & \num{78.72} & \pm & \num{ 1.32} & \num{11.5} &   & \num{  77.42} & \pm & \num{ 1.04} & \num{ 4.4} &   & \num{  77.84} & \pm & \num{ 1.31} & \num{11.0} \\
\bottomrule
   \end{tabular*}
\end{table}

%% file: ipBoost_arXiv.bbl
\begin{thebibliography}{10}

\bibitem{BarJNSV98}
C.~{Barnhart}, E.~L. {Johnson}, G.~L. {Nemhauser}, M.~W.~P. {Savelsbergh}, and
  P.~H. {Vance}.
\newblock Branch-and-price: Column generation for solving huge integer
  programs.
\newblock {\em Oper. Res.}, 46(3):316--329, 1998.

\bibitem{Berthold14}
T.~Berthold.
\newblock {\em Heuristic algorithms in global {MINLP} solvers}.
\newblock PhD thesis, TU Berlin, 2014.

\bibitem{bertsimas2017optimal}
D.~Bertsimas and J.~Dunn.
\newblock Optimal classification trees.
\newblock {\em Machine Learning}, pages 1--44, 2017.

\bibitem{bertsimas2015or}
D.~Bertsimas and A.~King.
\newblock {OR} forum---an algorithmic approach to linear regression.
\newblock {\em Operations Research}, 64(1):2--16, 2015.

\bibitem{bertsimas2016best}
D.~Bertsimas, A.~King, R.~Mazumder, et~al.
\newblock Best subset selection via a modern optimization lens.
\newblock {\em The Annals of Statistics}, 44(2):813--852, 2016.

\bibitem{bertsimas2007classification}
D.~Bertsimas and R.~Shioda.
\newblock Classification and regression via integer optimization.
\newblock {\em Operations Research}, 55(2):252--271, 2007.

\bibitem{BLRSW13}
R.~Bornd{\"o}rfer, A.~L{\"o}bel, M.~Reuther, T.~Schlechte, and S.~Weider.
\newblock Rapid branching.
\newblock {\em Public Transport}, 5(1):3--23, 2013.

\bibitem{BLW08}
R.~Bornd{\"o}rfer, A.~L{\"o}bel, and S.~Weider.
\newblock A bundle method for integrated multi-depot vehicle and duty
  scheduling in public transit.
\newblock In M.~Hickman, P.~Mirchandani, and S.~Vo{\ss}, editors, {\em
  Computer-aided Systems in Public Transport}, volume 600, pages 3--24, 2008.

\bibitem{Bre96}
L.~Breiman.
\newblock Bagging predictors.
\newblock {\em Machine Learning}, 24:123--140, 1996.

\bibitem{chang2012integer}
A.~Chang, D.~Bertsimas, and C.~Rudin.
\newblock An integer optimization approach to associative classification.
\newblock In {\em Advances in Neural Information Processing Systems}, pages
  269--277, 2012.

\bibitem{cortes2014deep}
C.~Cortes, M.~Mohri, and U.~Syed.
\newblock Deep boosting.
\newblock In {\em 31st International Conference on Machine Learning, ICML
  2014}. International Machine Learning Society (IMLS), 2014.

\bibitem{cplex}
CPLEX.
\newblock {IBM ILOG CPLEX Optimizer}.
\newblock \url{https://www.ibm.com/analytics/cplex-optimizer}, 2020.

\bibitem{dash2018boolean}
S.~Dash, O.~G{\"u}nl{\"u}k, and D.~Wei.
\newblock Boolean decision rules via column generation.
\newblock In {\em Advances in Neural Information Processing Systems}, pages
  4655--4665, 2018.

\bibitem{demiriz2002linear}
A.~Demiriz, K.~P. Bennett, and J.~Shawe-Taylor.
\newblock Linear programming boosting via column generation.
\newblock {\em Machine Learning}, 46(1-3):225--254, 2002.

\bibitem{desrosiers2005primer}
J.~Desrosiers and M.~E. L{\"u}bbecke.
\newblock A primer in column generation.
\newblock In {\em Column generation}, pages 1--32. Springer, 2005.

\bibitem{eckstein2012improved}
J.~Eckstein and N.~Goldberg.
\newblock An improved branch-and-bound method for maximum monomial agreement.
\newblock {\em INFORMS Journal on Computing}, 24(2):328--341, 2012.

\bibitem{fischetti2018deep}
M.~Fischetti and J.~Jo.
\newblock Deep neural networks and mixed integer linear optimization.
\newblock {\em Constraints}, 23(3):296--309, 2018.

\bibitem{freund2015new}
R.~M. Freund, P.~Grigas, and R.~Mazumder.
\newblock A new perspective on boosting in linear regression via subgradient
  optimization and relatives.
\newblock {\em arXiv preprint arXiv:1505.04243}, 2015.

\bibitem{freund1995desicion}
Y.~Freund and R.~E. Schapire.
\newblock A desicion-theoretic generalization of on-line learning and an
  application to boosting.
\newblock In {\em European Conference on Computational Learning Theory}, pages
  23--37. Springer, 1995.

\bibitem{GleixnerEtal2018OO}
A.~Gleixner, M.~Bastubbe, L.~Eifler, T.~Gally, G.~Gamrath, R.~L. Gottwald,
  G.~Hendel, C.~Hojny, T.~Koch, M.~E. L{\"u}bbecke, S.~J. Maher,
  M.~Miltenberger, B.~M{\"u}ller, M.~E. Pfetsch, C.~Puchert, D.~Rehfeldt,
  F.~Schl{\"o}sser, C.~Schubert, F.~Serrano, Y.~Shinano, J.~M. Viernickel,
  M.~Walter, F.~Wegscheider, J.~T. Witt, and J.~Witzig.
\newblock {The SCIP Optimization Suite 6.0}.
\newblock Technical report, Optimization Online, July 2018.

\bibitem{goldberg2010boosting}
N.~Goldberg and J.~Eckstein.
\newblock Boosting classifiers with tightened l0-relaxation penalties.
\newblock In {\em Proceedings of the 27th International Conference on Machine
  Learning (ICML-10)}, pages 383--390, 2010.

\bibitem{goldberg2012sparse}
N.~Goldberg and J.~Eckstein.
\newblock Sparse weighted voting classifier selection and its linear
  programming relaxations.
\newblock {\em Information Processing Letters}, 112(12):481--486, 2012.

\bibitem{gunluk2016optimal}
O.~G{\"u}nl{\"u}k, J.~Kalagnanam, M.~Menickelly, and K.~Scheinberg.
\newblock Optimal generalized decision trees via integer programming.
\newblock {\em arXiv preprint arXiv:1612.03225}, 2016.

\bibitem{gunluk2018optimal}
O.~G{\"u}nl{\"u}k, J.~Kalagnanam, M.~Menickelly, and K.~Scheinberg.
\newblock Optimal decision trees for categorical data via integer programming.
\newblock {\em arXiv preprint arXiv:1612.03225}, 2018.

\bibitem{gurobi}
Gurobi.
\newblock {Gurobi Optimizer}.
\newblock \url{http://www.gurobi.com}, 2020.

\bibitem{leskovec2003linear}
J.~Leskovec and J.~Shawe-Taylor.
\newblock Linear programming boosting for uneven datasets.
\newblock In {\em Proceedings of the Twentieth International Conference on
  Machine Learning (ICML-2003)}, pages 456--463, 2003.

\bibitem{long2008random}
P.~M. Long and R.~A. Servedio.
\newblock Random classification noise defeats all convex potential boosters.
\newblock In {\em Proceedings of the 25th international conference on Machine
  learning}, pages 608--615. ACM, 2008.

\bibitem{long2010random}
P.~M. Long and R.~A. Servedio.
\newblock Random classification noise defeats all convex potential boosters.
\newblock {\em Machine learning}, 78(3):287--304, 2010.

\bibitem{pedregosa2011scikit}
F.~Pedregosa, G.~Varoquaux, A.~Gramfort, V.~Michel, B.~Thirion, O.~Grisel,
  M.~Blondel, P.~Prettenhofer, R.~Weiss, V.~Dubourg, et~al.
\newblock Scikit-learn: Machine learning in python.
\newblock {\em Journal of Machine Learning Research}, 12(Oct):2825--2830, 2011.

\bibitem{Pfetsch2008}
M.~E. Pfetsch.
\newblock Branch-and-cut for the maximum feasible subsystem problem.
\newblock {\em SIAM J. Optim.}, 19(1):21--38, 2008.

\bibitem{savicky2000optimal}
P.~Savick{\`y}, J.~Klaschka, and J.~Antoch.
\newblock Optimal classification trees.
\newblock In {\em COMPSTAT}, pages 427--432. Springer, 2000.

\bibitem{schapire2003boosting}
R.~E. Schapire.
\newblock The boosting approach to machine learning: An overview.
\newblock In {\em Nonlinear estimation and classification}, pages 149--171.
  Springer, 2003.

\bibitem{shalev2016minimizing}
S.~Shalev-Shwartz and Y.~Wexler.
\newblock Minimizing the maximal loss: How and why.
\newblock In {\em Proceedings of the 32nd International Conference on Machine
  Learning}, 2016.

\bibitem{tjeng2017evaluating}
V.~Tjeng, K.~Xiao, and R.~Tedrake.
\newblock Evaluating robustness of neural networks with mixed integer
  programming.
\newblock {\em arXiv preprint arXiv:1711.07356}, 2017.

\bibitem{verwer2019learning}
S.~Verwer and Y.~Zhang.
\newblock Learning optimal classification trees using a binary linear program
  formulation.
\newblock In {\em 33rd AAAI Conference on Artificial Intelligence}, 2019.

\bibitem{xpress}
XPRESS.
\newblock {FICO Xpress Optimizer}.
\newblock \url{https://www.fico.com/en/products/fico-xpress-optimization},
  2020.

\bibitem{zhu2009multi}
J.~Zhu, H.~Zou, S.~Rosset, and T.~Hastie.
\newblock Multi-class adaboost.
\newblock {\em Statistics and its Interface}, 2(3):349--360, 2009.

\end{thebibliography}
